\documentclass[11pt,english,reqno]{article}

\usepackage{graphicx}
\usepackage{subcaption}
\usepackage{amssymb}
\usepackage{amsmath}
\usepackage{amsthm}
\usepackage{xcolor}
\usepackage{enumitem}
\usepackage[margin=1.3in]{geometry}
\usepackage{tocloft}
\usepackage[hidelinks]{hyperref}
\usepackage[textwidth=3cm]{todonotes}
\usepackage{comment}
\usepackage{yhmath}
\usepackage[T1]{fontenc}

\DeclareMathOperator*{\diag}{diag}
\DeclareMathOperator{\Tr}{Tr}

\DeclareMathOperator{\Id}{Id}

\DeclareMathOperator{\im}{im}

\DeclareMathOperator{\Sym}{Sym}

\DeclareMathAlphabet{\mathbbb}{U}{bbold}{m}{n}

\newcommand{\diff}{\mathrm{d}}

\newcommand{\E}{\mathbb{E}}
\newcommand{\R}{\mathbb{R}}

\renewcommand{\leq}{\leqslant}
\renewcommand{\geq}{\geqslant}

\theoremstyle{plain}
\newtheorem{proposition}{Proposition}[section]

\newtheorem{lem}[proposition]{Lemma}
\newtheorem{coro}[proposition]{Corollary}
\newtheorem{result}[proposition]{Result}

\theoremstyle{definition}

\theoremstyle{remark}

\newtheorem{remark}[proposition]{Remark}

\numberwithin{equation}{section}

\author{}

\begin{document}

\title{The Multiple Timescales of Gradient Descent on~the~Edge~of~Stability: A~Perturbative~Derivation~of~the~Central~Flow}

\author{Raphaël Berthier\\
\small
Sorbonne Université, Inria, Centre Inria de Sorbonne Université, Paris, France}

\maketitle

\begin{abstract}
    The central flow of Cohen et al.~(2025) is an empirically accurate continuous-time model of gradient descent at the edge of stability in deep learning, However, its derivation is heuristic. We propose a perturbative regime in which the central flow is the limit of gradient descent: we assume that the loss decomposes as $f = g + \varepsilon h$; in the limit $\varepsilon \to 0$, the dynamics of gradient descent with learning rate $\eta$ converge to the gradient flow of $h$ constrained to the minimizers of $g$ of sharpness at most $2/\eta$. Our approach is formal rather than rigorous; it treats gradient descent as a singularly perturbed dynamical system in $\varepsilon$. Three timescales emerge: a fast timescale of oscillations along the sharpest direction, an intermediate timescale of the self-stabilization mechanism, and a slow timescale of the dynamics along the minimizers of $g$---the central flow. Using the method of multiple scales, a classical formal method from singular perturbation theory, we derive the expansion of the dynamics in $\varepsilon$: the central flow emerges as the leading-order term in the expansion, while the self-stabilization mechanism appears in the next-order term. We study this mechanism beyond previous analyses: with a single eigenvalue at the edge of stability, we compute the slow drift of the energy of the fluctuations; with several eigenvalues at the edge of stability, we derive the self-stabilization system and explain why fluctuations persist.
\end{abstract}

\noindent

\tableofcontents

\section{Introduction}

\paragraph{Context.} Gradient descent in deep learning still lacks theoretical understanding. Recent works have pointed out a surprising phenomenology. A large part of the training occurs at the so-called \emph{edge of stability} \cite{jastrzkebski2018relation,jastrzebski2020break,cohen2021gradient}. In this regime, the \emph{sharpness} of the loss---the maximal eigenvalue of the Hessian---fluctuates around the value $2/\eta$, where $\eta$ is the learning rate. Above this $2/\eta$ threshold, gradient descent on the local quadratic approximation of the loss develops diverging oscillations along the sharpest direction of the loss; it is thus remarkable that gradient descent on the loss itself does not diverge when fluctuations increase the sharpness above $2/\eta$. This is due to a so-called \emph{self-stabilization} phenomenon~\cite{damian2022self}: 
the oscillations along the sharpest direction of the loss induce an effect in the local cubic approximation of the loss that pushes back the dynamics in the direction of sharpness reduction. 

Overall, the resulting iterates form a rough path, both due to the oscillations along the sharpest direction and the fluctuations of the self-stabilization mechanism. This indicates that the optimization path cannot be approximated by a smooth continuous-time flow such as the gradient flow of the loss. This is a critical problem as the gradient flow is largely used to model gradient descent in deep learning. 

Instead, one can expect a continuous-time approximation for the \emph{averaged} trajectory, where the averaging is performed over the oscillations along the sharpest direction, and potentially over the fluctuations of the self-stabilization mechanism. Averaging over both, Cohen et al.~\cite{cohen2024understanding} proposed the \emph{central flow}: the gradient flow of the loss constrained to the stable region where the sharpness is below~$2/\eta$. Through extensive numerical experiments, they have shown that the central flow is an excellent approximation of the averaged trajectory across a wide range of datasets and architectures. 

\paragraph{Problem.} However, the derivation of the central flow in \cite{cohen2024understanding} is heuristic. Several approximation steps are left informal, for instance using a local cubic approximation of the loss, or the meaning of the ``local time-averaging'' operator $\E$. This raises the following question: is there a regime in which the central flow is the limit of gradient descent? 

The regime $\eta\to0$ of vanishing learning rate, used to derive the gradient flow, is not appropriate to answer this question. In this regime, the sharpness is always below $2/\eta$, and thus edge of stability never occurs. 

\paragraph{Contributions.} In this work, we propose a perturbative regime---the \emph{sharp valley assumption}---from which we derive the self-stabilization mechanism and the central flow. Mathematically, we assume that we run gradient descent with learning rate $\eta$ on a loss function $f(z) = g(z) + \varepsilon h(z)$. The function $g$ defines a valley: it has several minimizers and we denote $V = \{z : g(z) = \min g\}$ the valley of its minimizers. Further, $\varepsilon \ll 1$ is a small parameter and $h$ is a perturbation that induces the dynamics along the valley. Equivalently, gradient descent 
with learning rate $\varepsilon \eta$ on the loss function $\varepsilon^{-1} g + h$ gives the same sequence of iterates. The parameter $\varepsilon$ can thus be interpreted as the inverse sharpness of the valley. 

The sharp valley picture has been present in several works \cite{xing2018walk, jastrzkebski2018relation, wen2024understanding,liu2025focus, liu2025neural,song2025does}, as it is consistent with several empirical observations, such as the anti-alignment of successive iterates or the approximate low-rank structure of the Hessian (see Sec.~\ref{sec:related} for a more complete discussion). A contribution of this work is to show that it is also consistent with the edge of stability, self-stabilization and the central flow.

Indeed, in the regime $\varepsilon \to 0$, the central flow is the formal limit of gradient descent. More precisely, in this regime, gradient descent is a singularly perturbed dynamical system in which three timescales emerge: a fast timescale $t_1 = k$, the iteration number $k$, which is the typical timescale of the oscillations along the sharpest direction; an intermediate timescale $t_2 = \varepsilon^{1/2} k$, which is the typical timescale of the self-stabilization mechanism; and a slow timescale $t_3 = \varepsilon k$, which is the typical timescale of the central flow---the evolution along the valley. 

We then use the \emph{method of multiple scales}~\cite{holmes2012introduction} to derive an approximation of gradient descent that decouples the three timescales $t_1$, $t_2$, and $t_3$. The leading term in the approximation is the central flow, which evolves on the slow timescale $t_3$. It consists in the gradient flow of $h$ in the valley, constrained to the stable region where the sharpness of $g$ is below $2/\eta$. The next-order term contains the oscillations along the sharpest direction as a function of $t_1$ and the self-stabilization fluctuations as a function of $t_2$. The method of multiple scales is a \emph{formal} method in the sense of asymptotic analysis: it produces the expansion by matching orders, without proving convergence or bounding the remainder. Our results should thus be interpreted as formal derivations, i.e., conjectures with strong supporting evidence.

To summarize, this work provides some justification of the central flow approximation by proposing a regime in which the approximation is tight and by providing an analytical tool---the method of multiple scales---to derive the approximation. 

\paragraph{Further consequences.} Beyond the clarification of previous works, the framework of this paper enables some new discussions.


We study the self-stabilization mechanism in more depth: we derive the self-stabilization system with multiple eigenvalues on the edge of stability (left open in \cite{damian2022self}), and discuss the evolution of the self-stabilization fluctuations on the slow timescale $t_3$. With a single eigenvalue on the edge of stability, we compute the variations of the energy of the fluctuations; with multiple eigenvalues on the edge of stability, we explain the persistence of fluctuations. 

Finally, we compare the central flow with competing continuous-time approximations of gradient descent \cite{regis2026rod, marion2026edge, de2026gradient}. We obtain that the central flow is a more coarse-grained approximation than the competing approximations, as it depends only on~$t_3$, while the other approximations depend on $t_2$ and $t_3$. 

\paragraph{Outline.} This paper is organized as follows. In Sec.~\ref{sec:setting}, we introduce our mathematical setting, including the sharp valley assumption. The strongest result of this paper is provided in Sec.~\ref{sec:general}; however, for the sake of clarity and to enable additional discussions, previous sections provide simpler versions of the result. In Sec.~\ref{sec:cf}, we present only the results on the convergence to the central flow. In Sec.~\ref{sec:minimal}, we include the discussion of the next-order term in the convergence, but in a minimal example. The simplicity in this example allows to study potential failures of the central flow approximation and the slow evolution of the self-stabilization loop. Further, in Sec.~\ref{sec:codim-one}, we discuss the case of valleys of codimension one, which is more general than the minimal example but still simpler than the general case. In Sec.~\ref{sec:general}, we discuss the general case of valleys of arbitrary codimension. Finally, in Sec.~\ref{sec:related}, we discuss the connections to further related work, including the comparison with competing approxiamtions.

\section{Setting}
\label{sec:setting}

Let $f(z)$, $z \in Z$, be a loss function, where $Z$ is a finite-dimensional vector space. We study the dynamics of gradient descent on $f$ with learning rate $\eta > 0$:
\begin{equation*}
    z(k+1) = z(k) - \eta \nabla f(z(k)) \, .
\end{equation*}
We assume that the loss function $f$ can be decomposed as $f(z) = g(z) + \varepsilon h(z)$, where $g$ defines a valley and $\varepsilon \ll 1$ is a small parameter. We denote $V = \{ z \in Z : g(z) = \min g\}$ the valley of minimizers of $g$. 


To simplify the analyses, we assume that the valley defined by $g$ is a linear space. More precisely, we assume there exists an orthogonal decomposition $Z = X \oplus^\perp Y$, in which we decompose $z = (x, y)$, such that the valley defined by $g$ is
\begin{equation*}
    V = \{(0,y), y \in Y\} \, . 
\end{equation*} 
In a more complex setting where $V$ would be a manifold, the decomposition $z = (x,y)$ would represent a local set of coordinates to parametrize the manifold $V$ of minimizers of $g$. 

Further, for all $z$, we denote $S(z) = \nabla^2_x g(z)$ the Hessian of $g$ in the sharp directions---those orthogonal to the valley. We assume that for all $(0,y) \in V$, $S(0, y) \succ 0$: there is some curvature of~$g$ in the sharp directions along the valley. Finally, we denote 
\begin{equation*}
    V_{\mathrm{S}} = \{ z \in V : \eta S(z) \preccurlyeq 2 \Id_X\} 
\end{equation*}
the \emph{stable} valley of $g$, in which the sharpness is below $2/\eta$. 

\paragraph{Notations.} When $y \in \R^d$ is a vector, we denote $y[1], \dots, y[d]$ its components. Let $A$ and $B$ be Euclidean spaces. When $L: A \to B$ is a linear operator, we denote $L^*: B \to A$ its adjoint. We denote $\Sym A$ the set of symmetric linear operators on $A$, and $\Sym_+ A$ the set of symmetric positive semidefinite linear operators on $A$. We denote $\Id_A$ the identity operator on $A$. When $C$ is a linear subspace of $A$, we denote $C^\perp$ its orthogonal complement, $P_C : A \to C$ the orthogonal projection onto $C$, and $\iota_C : C \to A$ the canonical inclusion of $C$ into $A$. Note that $\iota_C = P_{C}^*$.

\section{Convergence to the central flow}
\label{sec:cf}

In this section, we state our convergence results to the central flow. More complete results, including next-order terms, are provided in the following sections.

As the gradient of $f = g + \varepsilon h$ in the valley is of order $\varepsilon$, the typical timescale of the dynamics of the iterates $z(k) = (x(k), y(k))$ along the valley is $t_3 = \varepsilon k$. Abusing notations, let us denote 
\begin{align*}
    &z(t_3) = z\left(k=\lfloor t_3/\varepsilon \rfloor\right) \, , &&x(t_3) = x\left(k=\lfloor t_3/\varepsilon \rfloor\right) \, , &&y(t_3) = y\left(k=\lfloor t_3/\varepsilon \rfloor\right) \, ,
\end{align*}
the time-rescaled iterates, that now depend on a continuous time parameter $t_3$. Note that the resulting function $z(t_3)$ depends on $\varepsilon$ both through the rescaling of time and through the loss function. We now state convergence results to the central flow, obtained by the method of multiple scales.

\paragraph{Valleys of codimension one.}For the sake of clarity, we start with the case where the valley is of codimension $1$, i.e., the dimension of $X$ is $1$. Note that, in this case, $S(z)$ is a scalar.

\begin{result}
    \label{result:conv-cf}
    Assume that $X$ is of dimension $1$. Further, assume that as $\varepsilon \to 0$, the rescaled iterates $z(t_3) = (x(t_3), y(t_3))$ converge to some function $z_0(t_3) = (x_0(t_3), y_0(t_3))$. Moreover, assume that $z_0$ remains in the stable valley $V_{\mathrm{S}}$, i.e., $x_0 = 0$ and $\eta S(0, y_0) \leq 2$. 

    Then $z_0$ follows the gradient flow of $h$ constrained to the stable valley $V_{\mathrm{S}}$, i.e., there exists $\sigma^2 = \sigma^2(t_3) \geq 0$ such that 
    \begin{align*}
        \partial_{t_3} y_0 &= - \eta \left[\nabla_y h(0, y_0) +\frac{\sigma^2}{2} \nabla_y S(0, y_0) \right] \, , 
    \end{align*}
    where for all $t_3$, $\sigma^2(t_3) = 0$ or $\eta S(0, y_0(t_3)) = 2$. 
\end{result}

Again, the method of multiple scales is formal, thus the above result should not be interpreted as a mathematical theorem but rather as a conjecture with strong supporting evidence. All the results of this paper share this status. An introduction to the method of multiple scales is provided in Sec.~\ref{sec:mms}.

The function $\sigma^2$ should be interpreted as the Lagrange multiplier associated with the constraint $\eta S(0, y_0) \leq 2$: it satisfies dual feasibility ($\sigma^2 \geq 0$) and complementarity slackness with the constraint ($\sigma^2 = 0$ or $\eta S(0, y_0) = 2$). 

We illustrate the result in Fig.~\ref{fig:codim-one-proposal}. We observe the edge of stability phenomenon, the self-stabilization fluctuations, and the convergence to the central flow $(x_0, y_0)$ as $\varepsilon \to 0$. A more pedagogical example is provided in the next section.

\begin{figure}[ht]
    \centering
        \includegraphics[width=\linewidth]{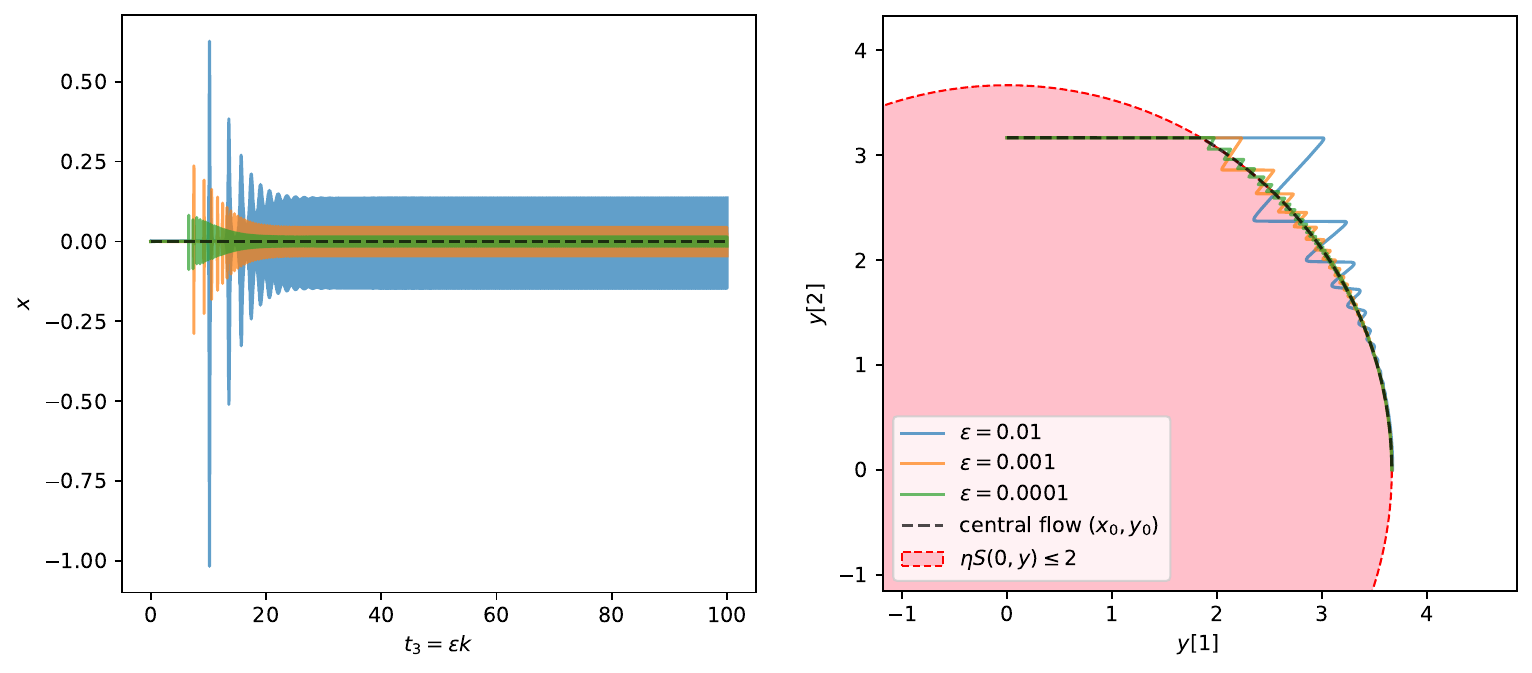}
    \caption{Iterates of gradient descent with learning rate $\eta = 0.3$ on $f = g + \varepsilon h$, with $\varepsilon = 0.01, 0.001, 0.0001$. In this example, we take $Z = \R^3$, $X = \R$ and $Y = \R^2$, thus the valley is of dimension $2$ and codimension $1$. Further, we take $g(x,y) = \frac{1}{2}(\Vert y \Vert + 3 e^{x/2}) x^2$ and $h(x,y) = -x - y[1]$ (where $y[1]$ denotes the first component of $y$). Note that, as required by Sec.~\ref{sec:setting}, the valley of minimizers of $g$ is $V = \{(0,y), y \in \R^2\}$. Moreover, $S(0, y) = \Vert y \Vert + 3$ and thus $V_{\mathrm{S}}$ is a closed disk: $V_{\mathrm{S}} = \{(0,y), \Vert y \Vert \leq 2/\eta-3\}$. }
    \label{fig:codim-one-proposal}
\end{figure}

We now discuss the assumptions of Result~\ref{result:conv-cf}. The statement $x_0 = 0$ is not strictly needed and could instead be obtained as a consequence of the method of multiple scales. However, treating $x_0$ in the method derivations results in more complicated computations, thus we prefer making the reasonable assumption that the limiting dynamics are in the valley $V$. Moreover, note that we could initialize $x$ at a non-zero value, which would result in a relaxation of $x_0$ to $0$ on the timescale $t_1$. This early part of the dynamics is not interesting for our study of the edge of stability, thus we do not include it here. 

On the contrary, we need the assumption $\eta S(0,y_0) \leq 2$ to derive the result; we do not know how to obtain it as a by-product of the method. In fact, it is possible to build artificial examples where this condition breaks and the central flow prediction fails; this is discussed below in Remark \ref{rmk:cf-failure}. 

\paragraph{Valleys of arbitrary codimension.}We now turn to the more general case where the dimension of $X$ is arbitrary. Recall that we denote $S(z) = \nabla^2_x g(z) \in \Sym X$. Further, we denote 
\begin{align*}
        &D_y S(x,y) : Y \to \Sym X  \\
    &\hspace{5.5em} \delta y \mapsto D_y S(x,y)[\delta y] 
\end{align*}
the differential of $S$ in the $y$ variable, and consequently, $D_y S(x,y)^* : \Sym X \to Y$ the adjoint of $D_y S(x,y)$. Finally, we denote $\langle ., . \rangle$ the canonical inner product on $\Sym X$. 

\begin{result}
    \label{result:conv-cf-general}
    We no longer assume that $X$ is of dimension~$1$. Further, assume that as $\varepsilon \to 0$, the rescaled iterates $z(t_3) = (x(t_3), y(t_3))$ converge to some function $z_0(t_3) = (x_0(t_3), y_0(t_3))$. Moreover, assume that $z_0$ remains in the stable valley $V_{\mathrm{S}}$, i.e., $x_0 = 0$ and $\eta S(0, y_0) \preccurlyeq 2 \Id_X$. 

    Then $z_0$ follows the gradient flow of $h$ constrained to the stable valley $V_{\mathrm{S}}$, i.e., there exists $\Sigma = \Sigma(t_3) \in \Sym_+ X$ such that 
    \begin{align*}
        \partial_{t_3} y_0 &= - \eta \left[\nabla_y h(0, y_0) +\frac{1}{2} D_y S(0, y_0)^*[\Sigma] \right] \, , 
    \end{align*}
    where for all $t_3$, $\langle \Sigma(t_3), 2 \Id_X - \eta S(0, y_0(t_3)) \rangle = 0$. 
\end{result}

Note that Result~\ref{result:conv-cf-general} is a generalization of Result~\ref{result:conv-cf}. The function $\Sigma$ should be interpreted as the Lagrange multiplier associated with the constraint $\eta S(0, y_0) \preccurlyeq 2 \Id_X$: it satisfies dual feasibility ($\Sigma \in \Sym_+ X$) and complementarity slackness with the constraint ($\langle \Sigma, 2 \Id_X - \eta S(0, y_0) \rangle = 0$).

We illustrate the result in Fig.~\ref{fig:general}.

\begin{figure}[htp]
    \centering
        \includegraphics[width=\linewidth]{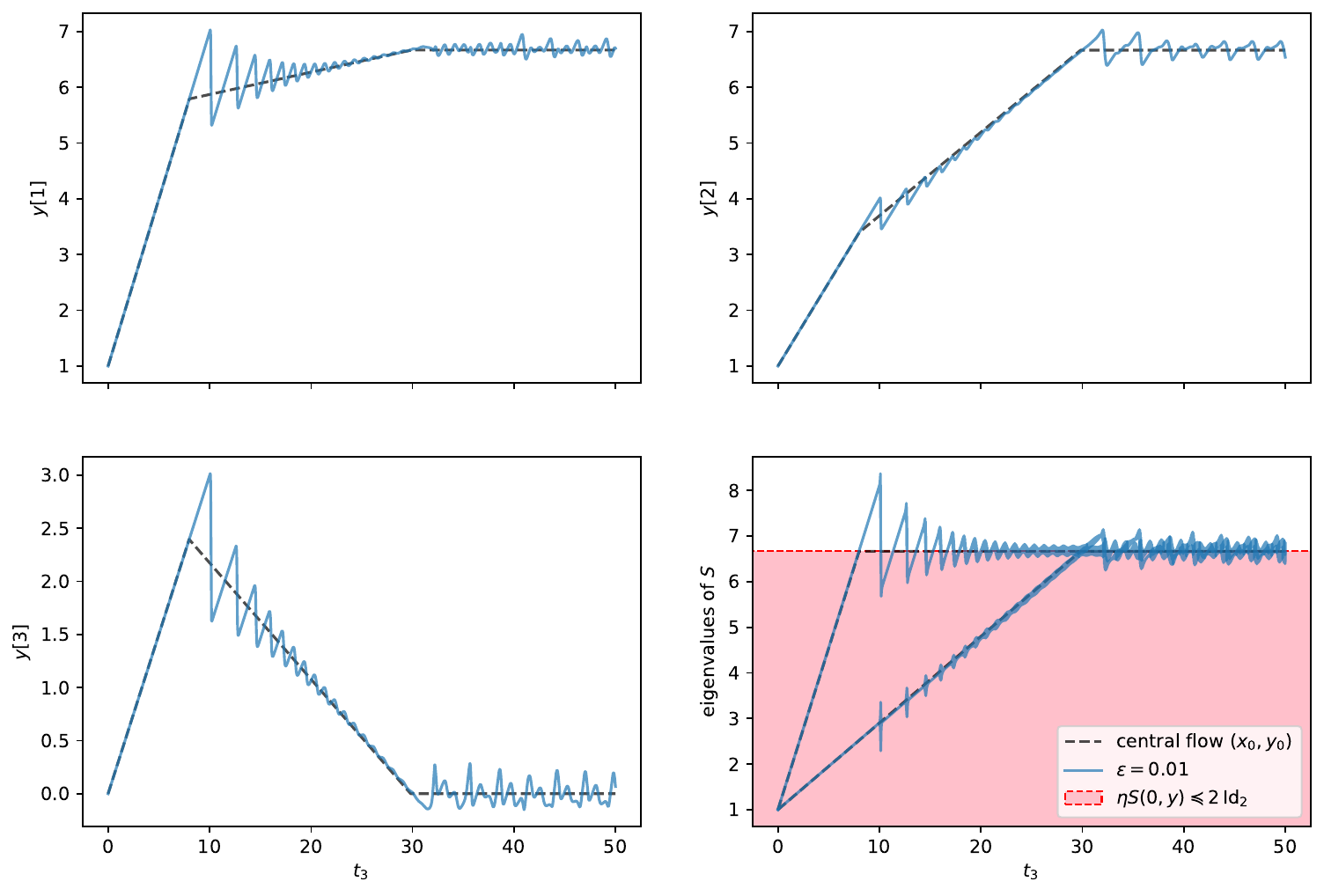}
    \caption{Iterates of gradient descent with learning rate $\eta = 0.3$ on $f = g + \varepsilon h$, with $\varepsilon = 0.01$. In this example, we take $Z = \R^5$, $X = \R^2$ and $Y = \R^3$, thus the valley is of dimension~$3$ and codimension $2$. Further, we take $g(x,y) = \frac{1}{2}x^\top H[y] x + \frac{1}{2}(x[1]^3 + x[2]^3)$
    with $H[y] = \begin{pmatrix}
    y[1] & \frac{y[3]}{\sqrt{2}} \\
    \frac{y[3]}{\sqrt{2}} & y[2] 
    \end{pmatrix}$and $h(x,y) = -x[1] - 2y[1]-y[2]-y[3]$. Note that the assumptions of Sec.~\ref{sec:setting} are satisfied in the region where $H[y] + \diag(x[1], x[2]) 
    \succ 0$ (where the iterates remain).
Moreover, $S(0, y) = H[y] $ and $V_{\mathrm{S}} = \left\{(0,y), \eta H[y] \preccurlyeq 2 \Id_2\right\}$.}
    \label{fig:general}
\end{figure}

The method of multiple scales also provides information on the oscillations and the self-stabilization fluctuations through the next-order term $z_1 = (x_1, y_1)$ in the expansion as $\varepsilon \to 0$. These complements are provided in Sec.~\ref{sec:codim-one} for valleys of codimension one and in Sec.~\ref{sec:general} for the general case. These aspects require discussing in more detail how the method of multiple scales works. For pedagogical purposes, we start doing so on a minimal example in the next section.

\section{Minimal example}
\label{sec:minimal}

\subsection{Introduction and experiments}

Consider the case $Z = \R^2$, $X = Y = \R$, thus $z = (x,y) \in \R^2$. Further, we consider the following choices determining the loss function $f = g + \varepsilon h$: 
\begin{align}
    &g(x,y) = \frac{1}{2} (y+\zeta x) x^2 \, , &&h(x,y) = -x-y \, . \label{eq:minimal}
\end{align}
Here, $\zeta \in \R$ is a parameter. This function $g$ does not rigorously satisfy the valley structure required in Sec.~\ref{sec:setting}. However, this is benign: below, we consider the function $g$ only in the region where $y+\zeta x > 0$. In this region, the minimum of $g$ is $0$ and its valley of minimizers is $V = \{(0,y), y > 0\}$. In the valley, $h$ decreases as $y$ increases, thus the gradient descent dynamics move along the valley in the direction of increasing~$y$. However, the sharpness of $g$ in the valley is $S(0,y) = \partial_{x}^2g(0,y) = y$, thus there is progressive sharpening, until the dynamics hit the edge of stability at $y = 2/\eta$. There are then self-stabilization fluctuations around the value $y = 2/\eta$. These phenomena are illustrated in Fig.~\ref{fig:minimal-example}.

\begin{figure}[ht]
    \centering
    \includegraphics[trim=128bp 27bp 118bp 71bp, clip, height=8cm]{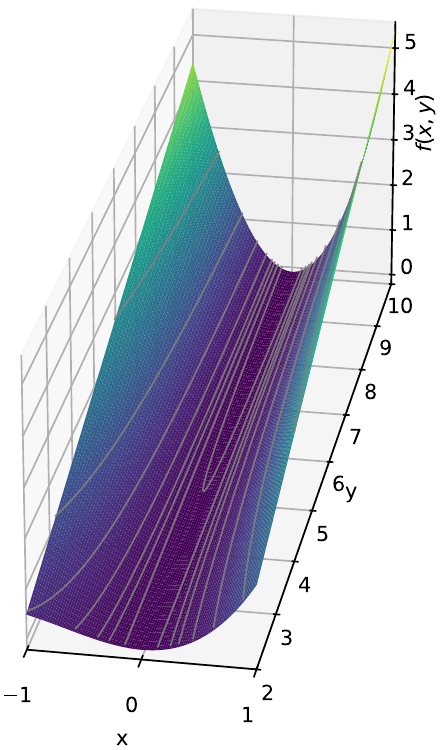}%
    \hspace{1.2cm}%
    \includegraphics[trim=8bp 8bp 6bp 6bp, clip, height=8cm]{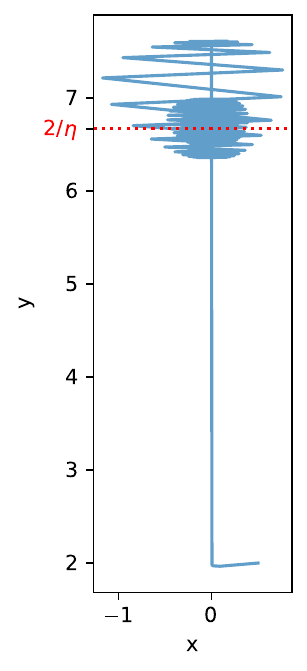}
    \caption{For $\varepsilon = 0.01$, $\zeta = 1$ and $g, h$ as provided in the minimal example of Eqs.~\eqref{eq:minimal}, we plot the loss function $f = g + \varepsilon h$ (left) and the iterates of gradient descent on $f$ with learning rate $\eta = 0.3$ (right). The iterates stabilize on the edge of stability $y = 2/\eta$ after some self-stabilization fluctuations. 
    \label{fig:minimal-example}}
\end{figure}

The reader might object that the above example is not the simplest possible one with the same properties. For instance, one could consider $g(x,y) = \frac{1}{2} y x^2$ and $h(x,y) = -y$. However, it turns out that these choices would not lead to the exact same phenomenology, and the ones we make lead to a more faithful reproduction of the phenomenology of deep learning optimization. In short, the $x^3$ term in $g$ reproduces the damping of the self-stabilization loop, and the $-x$ term in $h$ breaks the metastability in $x$. These aspects are discussed in Sec.~\ref{sec:energy} and Remark~\ref{rmk:cf-failure} respectively.

In Figure~\ref{fig:minimal-example-2}, we plot the iterates of gradient descent as a function of $t_3$, for various values of $\varepsilon$. We observe that the iterates converge to the central flow $(x_0, y_0)$ as $\varepsilon \to 0$, that is, to the gradient flow of $h$ constrained to the stable valley $V_{\mathrm{S}} = \{(0,y), y \leq 2/\eta\}$.  

\begin{figure}[ht]
    \centering
        \includegraphics[width=\linewidth]{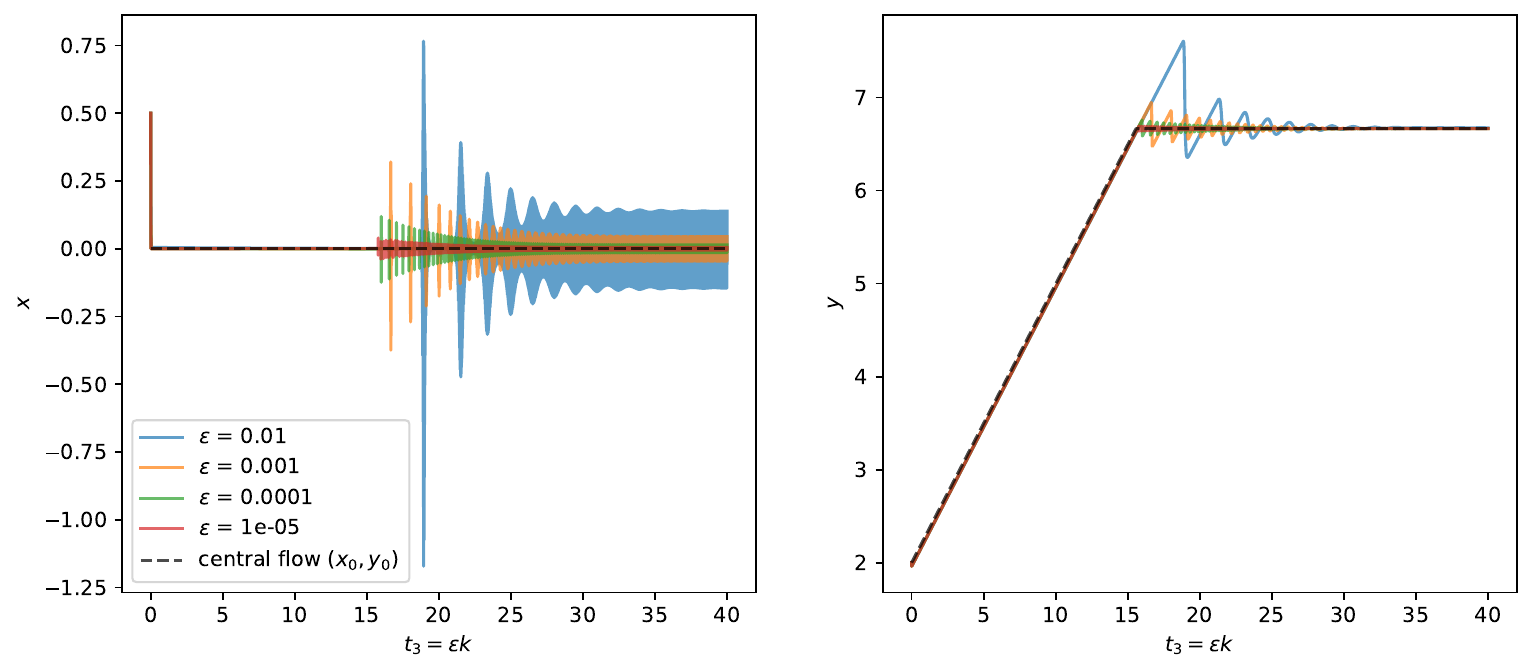}
    \caption{Iterates of gradient descent on the minimal example of Eqs.~\eqref{eq:minimal} with learning rate $\eta = 0.3$, $\zeta = 1$ and $\varepsilon = 10^{-2}, 10^{-3},  10^{-4}, 10^{-5}$.}
    \label{fig:minimal-example-2}
\end{figure}

At the edge of stability, the oscillations of $x$ and the self-stabilization fluctuations appear in the next-order term of the approximation. As will be confirmed by the method of multiple scales below, this next order term is of order $\varepsilon^{1/2}$. In Figure~\ref{fig:minimal-example-3}, we plot the rescaled iterates $\varepsilon^{-1/2} (x - x_0, y - y_0)$ as functions of the slow timescale $t_3 = \varepsilon k$ and the intermediate timescale $t_2 = \varepsilon^{1/2} k$. In the limit $\varepsilon \to 0$, we observe that the self-stabilization loop and its slow damping decouple on the timescales $t_2$ and $t_3$. While the typical time of a self-stabilization fluctuation is $t_2$, the typical time of the damping of the self-stabilization loop is $t_3$.

\begin{figure}[htp]
    \centering
    \includegraphics[width=\linewidth]{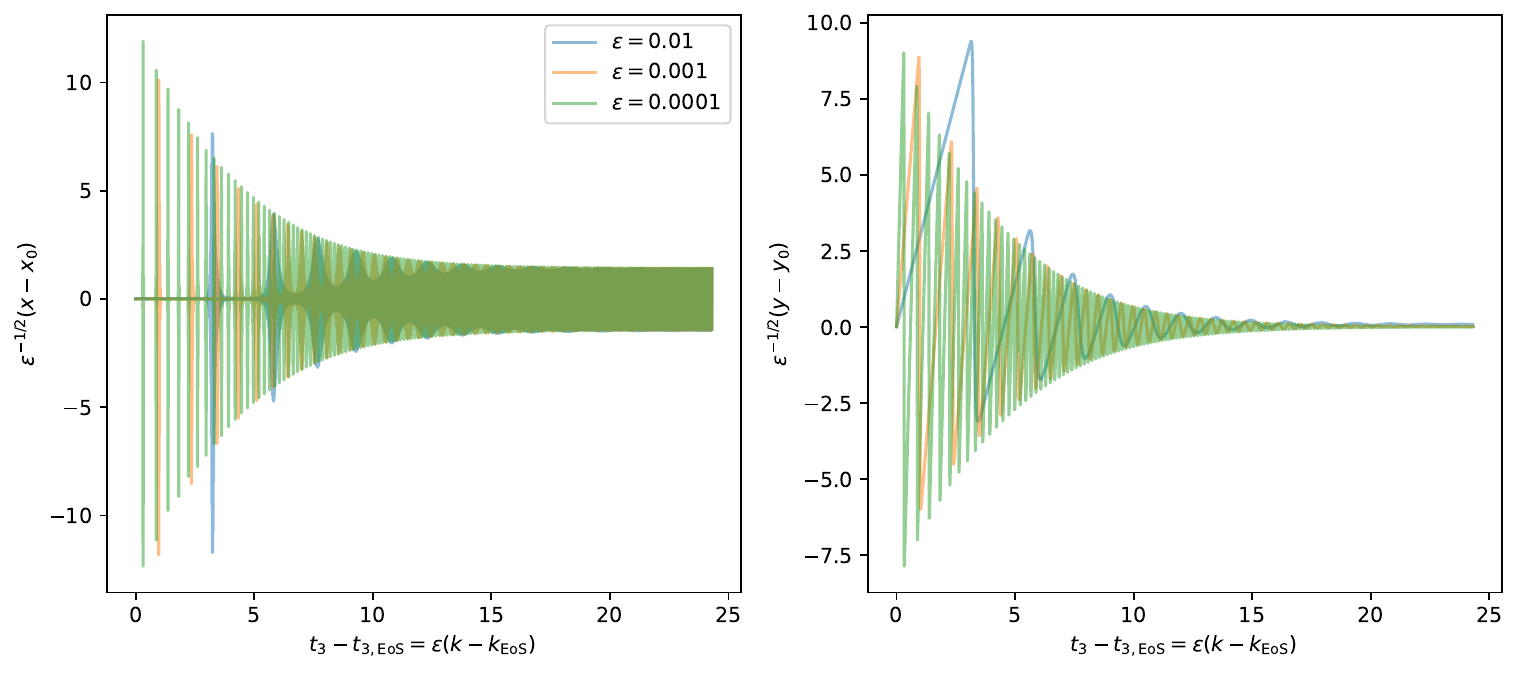}
        \includegraphics[width=\linewidth]{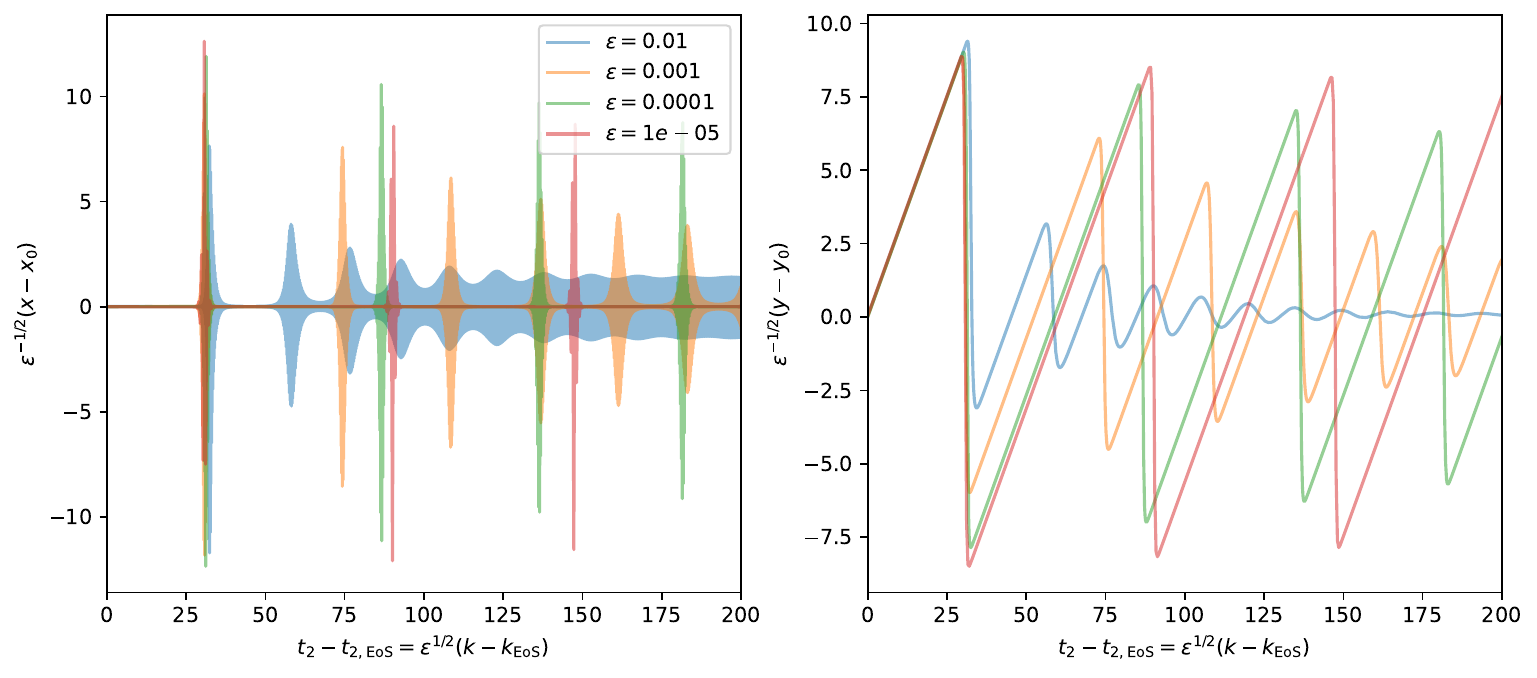}
    \caption{Rescaled iterates of gradient descent on the minimal example of Eqs.~\eqref{eq:minimal} with $\eta = 0.3$, $\zeta = 1$ and $\varepsilon = 10^{-2}, 10^{-3},  10^{-4}, 10^{-5}$. We denote $k_{\text{EoS}}$ the first iteration number for which $y > 2/\eta$: it encodes the beginning of the edge of stability. We denote $t_{2,\text{EoS}} = \varepsilon^{1/2}k_{\text{EoS}}$ and $t_{3,\text{EoS}} = \varepsilon k_{\text{EoS}}$ the corresponding rescaled times, and we start the plots at these times. We do not plot the trajectory with $\varepsilon = 10^{-5}$ in the upper plots for readability. }
    \label{fig:minimal-example-3}
\end{figure}

\subsection{Method of multiple scales}
\label{sec:mms}

The goal of this section is to confirm and refine the empirical observations above using the method of multiple scales. It is a perturbation method: in our case, the perturbation parameter is $\varepsilon \ll 1$, and it seeks an expansion of the form
\begin{align}
    \label{eq:minimal-expansion}
    \begin{split}
    x &= x_0 + \varepsilon^{1/2} x_1 + \varepsilon x_2 + \dots \, , \\
    y &= y_0 + \varepsilon^{1/2} y_1 + \varepsilon y_2 + \dots \, .
    \end{split}
\end{align}
The goal is to obtain expressions or evolution equations for the leading terms of this expansion. Here, the challenge is that the terms depend simultaneously on various timescales $t_1 = k$, $t_2 = \varepsilon^{1/2} k$, and $t_3 = \varepsilon k$, that depend themselves on $\varepsilon$. This is a typical situation where the method of multiple scales is useful.

Let us first give an overview of the method. The method postulates that the timescales $t_1$, $t_2$, and $t_3$ can be decoupled by treating them as independent parameters. In our case, $t_1$ is a discrete parameter while $t_2$ and $t_3$ are continuous parameters. It then treats all functions of $k$ as functions of the three independent timescales: $z = z(t_1, t_2, t_3)$, $x = x(t_1, t_2, t_3)$, $y = y(t_1, t_2, t_3)$, $x_0 = x_0(t_1, t_2, t_3)$, etc. The difference equation $\Delta_k z = z(k+1) - z(k) =  - \eta \nabla f(z)$ in $z = z(k)$ is transformed into a partial difference-differential equation in $z = z(t_1, t_2, t_3)$. By identifying the different terms in the expansion of this equation, we infer properties for the terms $x_0 = x_0(t_1, t_2, t_3)$, $y_0 = y_0(t_1, t_2, t_3)$, $x_1 = x_1(t_1, t_2, t_3)$, etc. The final approximation is provided by using the expansion of Eq.~\eqref{eq:minimal-expansion} and substituting the actual timescales $t_1 = k$, $t_2 = \varepsilon^{1/2} k$, and $t_3 = \varepsilon k$.

We obtain the following result for our minimal example.

\begin{result}
    \label{result:minimal}
Assume that $(x_0, y_0)$ remains in the stable valley, i.e., $x_0 = 0$ and $\eta y_0 \leq 2$. Then $y_0 = y_0(t_3)$, i.e., $y_0$ does not depend on $t_1$ or $t_2$; and $y_1 = y_1(t_2, t_3)$, i.e., $y_1$ does not depend on $t_1$. Further, we have the following:
\begin{enumerate}[label=(\roman*)]
    \item\label{it:interior} If $\eta y_0 <2$ (relative interior of the stable valley), then $\partial_{t_3} y_0 = - \eta \partial_y h(x_0, y_0) = \eta$. 
    \item\label{it:eos} If $\eta y_0 = 2$ (edge of stability), then $x_1 = (-1)^{t_1}a_1$ for some $a_1 = a_1(t_2, t_3)$. Further, we have the self-stabilization system 
    \begin{equation}
        \label{eq:self-stab-minimal}
        \begin{split}
            \partial_{t_2} a_1 &= \eta a_1 y_1 \, , \\
            \partial_{t_2} y_1 &= \eta \left(1 - \frac{a_1^2}{2} \right) \, . 
        \end{split}
    \end{equation}
\end{enumerate}
\end{result}
This result confirms that $(x_0, y_0)$ follows the central flow prediction. It moves in the stable valley $V_{\mathrm{S}} = \{(0, y) : \eta y \leq 2\}$, following the constrained gradient flow of~$h$, until it hits the constraint $\eta y_0 \leq 2$. It then stays on the edge of stability threshold. This is consistent with the experiment of Fig.~\ref{fig:minimal-example-2}.

This result also gives information about the next-order term $(x_1, y_1)$ at the edge of stability. While $y_1$ does not depend on $t_1$, $x_1$ oscillates in $t_1$, with amplitude $a_1(t_2, t_3)$. Again, this is consistent with the experiments of Figs.~\ref{fig:minimal-example-2} and \ref{fig:minimal-example-3}. Moreover the joint evolution of $a_1$ and $y_1$ in $t_2$ is provided by the system of ordinary differential equations (ODEs) \eqref{eq:self-stab-minimal}. This corresponds to the self-stabilization equations derived in \cite{damian2022self}. 

As noted in \cite{damian2022self}, these self-stabilization dynamics preserve the energy 
\begin{equation*}
    E = \frac{y_1^2}{2} + \frac{a_1^2}{4} - \frac{1}{2} \log a_1^2 \, , 
\end{equation*}
resulting in $(a_1, y_1)$ being periodic in $t_2$. As a consequence, $E = E(t_3)$, i.e., $E$ evolves only on the slow timescale $t_3$. This is consistent with the experiment of Fig.~\ref{fig:minimal-example-3}. The evolution of $E$ in $t_3$ is discussed in Sec.~\ref{sec:energy}. 

We now turn to the derivation of Result~\ref{result:minimal}. The situation in which we apply the method of multiple scales is rather sophisticated, due to the presence of three timescales (two is more common) and the presence of a discrete timescale $t_1$. For more pedagogical illustrations of the method, we recommend \cite{holmes2012introduction}.

\begin{proof}[Derivation of Result~\ref{result:minimal}]
As written above, the method of multiple scales operates by identifying the several scales in the gradient descent iteration $\Delta_k z = - \eta \nabla f(z)$. We thus expand both sides as functions of $\varepsilon$. For the right-hand side, we have
\begin{align}
    -\eta \nabla f(z) &= -\eta \nabla g(z) - \eta \varepsilon \nabla h(z) \nonumber\\
    &= -\eta \left[\nabla g\left(z_0 + \varepsilon^{1/2} z_1 + \varepsilon z_2 + o(\varepsilon)\right) + \varepsilon \nabla h\left(z_0 + o(1)\right) \right] \nonumber\\
    &= -\eta \bigg[\nabla g(z_0) + \nabla^2 g(z_0) \left(\varepsilon^{1/2} z_1 + \varepsilon z_2\right) + \frac{1}{2} \nabla^3 g(z_0) \left[\varepsilon^{1/2} z_1 , \varepsilon^{1/2} z_1 \right] \nonumber\\
    &\qquad\quad+ \varepsilon \nabla h (z_0) + o (\varepsilon) 
    \bigg] \nonumber\\
    &= -\eta \nabla g(z_0) -  \varepsilon^{1/2} \eta \nabla^2 g(z_0) z_1 - \varepsilon\eta \left[\nabla^2 g(z_0) z_2 + \frac{1}{2} \nabla^3 g(z_0) [z_1, z_1] + \nabla h(z_0)\right] \nonumber\\
    &\qquad
    + o(\varepsilon) \label{eq:rhs-general}\, .
\end{align}
Using the expressions for $g$ and $h$ of Eqs.~\eqref{eq:minimal} and the assumption $x_0 = 0$, we obtain 
\begin{align}
    \label{eq:rhs-minimal}
    \begin{split}
    -\eta \partial_x f(x,y) &= - \varepsilon^{1/2} \eta y_0 x_1 - \varepsilon \eta \left[y_0 x_2 + \frac{3}{2} \zeta x_1^2 +  x_1 y_1 - 1\right] + o(\varepsilon) \, , \\
    -\eta \partial_y f(x,y) &= - \varepsilon \eta \left[\frac{x_1^2}{2} - 1\right] + o(\varepsilon) \, .
    \end{split}
\end{align}
The expansion of the left-hand side $\Delta_k z$ is more subtle because it requires to express the discrete difference operator $\Delta_k$ applied to a function $z = z(t_1, t_2, t_3)$. As $t_1$, $t_2$, and~$t_3$ depend on $k$, the expression of $\Delta_k$ depends on the discrete difference operator~$\Delta_{t_1}$ and the partial derivatives $\partial_{t_2}$ and $\partial_{t_3}$.
\begin{align}
    \Delta_k z &= z(t_1 + 1, t_2 + \varepsilon^{1/2}, t_3 + \varepsilon) - z(t_1, t_2, t_3) \nonumber\\
    &= z(t_1 + 1, t_2 + \varepsilon^{1/2}, t_3 + \varepsilon) - z(t_1 , t_2 + \varepsilon^{1/2}, t_3 + \varepsilon) \nonumber\\
    &\qquad+ z(t_1 , t_2 + \varepsilon^{1/2}, t_3 + \varepsilon) - z(t_1, t_2, t_3) \nonumber\\
    &= (\Delta_{t_1} z)(t_1, t_2 + \varepsilon^{1/2}, t_3 + \varepsilon) + z(t_1 , t_2 + \varepsilon^{1/2}, t_3 + \varepsilon) - z(t_1, t_2, t_3) \label{eq:aux-4}
\end{align}
We first decompose 
\begin{align*}
    &z(t_1 , t_2 + \varepsilon^{1/2}, t_3 + \varepsilon) - z(t_1, t_2, t_3) = \varepsilon^{1/2} \partial_{t_2} z + \varepsilon \partial_{t_3} z + \frac{\varepsilon}{2} \partial_{t_2}^2 z + o(\varepsilon) \, .
\end{align*}
Note that we do not write the evaluations of the partial derivatives at $(t_1, t_2, t_3)$ for readability. Using a similar expansion for $(\Delta_{t_1} z)(t_1, t_2 + \varepsilon^{1/2}, t_3 + \varepsilon)$, we obtain 
\begin{align*}
    \Delta_k z &= \Delta_{t_1} z + \varepsilon^{1/2} \partial_{t_2} z + \varepsilon \partial_{t_3} z + \frac{\varepsilon}{2} \partial_{t_2}^2 z + \varepsilon^{1/2} \partial_{t_2} \Delta_{t_1} z + \varepsilon \partial_{t_3} \Delta_{t_1}z + \frac{\varepsilon}{2} \partial_{t_2}^2 \Delta_{t_1}z +o(\varepsilon) \\
    &= \Delta_{t_1} z + \varepsilon^{1/2} \left[\partial_{t_2} z + \partial_{t_2} \Delta_{t_1} z\right] + \varepsilon \left[\partial_{t_3} z + \frac{1}{2} \partial_{t_2}^2 z + \partial_{t_3} \Delta_{t_1}z + \frac{1}{2} \partial_{t_2}^2 \Delta_{t_1}z\right] + o(\varepsilon) \, . 
\end{align*}
We combine the above decomposition with the expansion of $z$ 
\begin{equation*}
    z = z_0 + \varepsilon^{1/2} z_1 + \varepsilon z_2 +  o(\varepsilon) \, .
\end{equation*}
We obtain 
\begin{align}
    \label{eq:lhs}
    \begin{split}
    \Delta_k z &= \Delta_{t_1} z_0 + \varepsilon^{1/2} \left[\Delta_{t_1} z_1 + \partial_{t_2} z_0 +  \partial_{t_2} \Delta_{t_1}z_0 \right] \\
    &\quad+ \varepsilon \left[\Delta_{t_1} z_2 + \partial_{t_2} z_1 + \partial_{t_2} \Delta_{t_1} z_1 + \partial_{t_3} z_0 + \frac{1}{2} \partial_{t_2}^2 z_0  + \partial_{t_3} \Delta_{t_1} z_0 + \frac{1}{2} \partial_{t_2}^2 \Delta_{t_1} z_0\right] \\
    &\quad+ o(\varepsilon) \, .
    \end{split}
\end{align}
Combining the expansions of Eqs.~\eqref{eq:rhs-minimal} and \eqref{eq:lhs}, we are ready to identify terms in the expansion of 
\begin{align}
    \Delta_k x &= -\eta \partial_x f(x,y) \, ,  \label{eq:x}\\
    \Delta_k y &= -\eta \partial_y f(x,y) \, . \label{eq:y}
\end{align}

\paragraph{Order $1$.} Recall that, by assumption, $x_0 = 0$. The $x$ equation \eqref{eq:x} does not give anything at order $1$. The $y$~equation~\eqref{eq:y} gives $\Delta_{t_1} y_0 = 0$, thus $y_0 = y_0(t_2, t_3)$ does not depend on $t_1$.

\paragraph{Order $\varepsilon^{1/2}$.} The $x$ equation \eqref{eq:x} gives 
\begin{equation*}
    \Delta_{t_1} x_1 = - \eta y_0 x_1 \, .
\end{equation*}
Here, $y_0$ should be interpreted as the sharpness of $g$ at $(x_0, y_0)$; this difference equation describes the linearized dynamics of $x$ around $(x_0, y_0)$. Its behavior depends on whether the edge of stability is reached. Fix $t_2$ and $t_3$. 
\begin{itemize}
    \item If $\eta y_0(t_2, t_3) < 2$, then $x_1 = x_1(t_1,t_2, t_3) \xrightarrow[t_1 \to \infty]{} 0$. 
    \item If $\eta y_0(t_2, t_3) = 2$, then $x_1 = (-1)^{t_1} a_1(t_2, t_3)$ for some $a_1$ that does not depend on~$t_1$. At the edge of stability, $x_1$ oscillates with an amplitude which is a function of $t_2$ and~$t_3$ only. 
\end{itemize}
We now turn to the $y$ equation \eqref{eq:y}. It gives
\begin{equation*}
    \Delta_{t_1} y_1 + \partial_{t_2} y_0 = 0 \, . 
\end{equation*}
Recall that $y_0$ does not depend on $t_1$. As a consequence, $y_1$ is an affine function of $t_1$, with slope $-\partial_{t_2} y_0$. However, having $y_1$ unbounded is not acceptable as it would imply that $\varepsilon^{1/2} y_1$ eventually become of order $1$, thus interfering with the leading order term $y_0$ in the expansion \eqref{eq:minimal-expansion}. In the method of multiple scales, such a term is called \emph{secular}, and the method prescribes to build an approximation that forbids secular terms, to keep the reasoning consistent \cite{holmes2012introduction}. This prescription is the core mechanism of the method of multiple scales; we will use it repeatedly. Here, we can thus conclude that $\Delta_{t_1} y_1 = -\partial_{t_2} y_0 = 0$. As a consequence, $y_1 = y_1(t_2, t_3)$ and $y_0 = y_0(t_3)$.

\paragraph{Order $\varepsilon$.} The $x$ equation \eqref{eq:x} gives
\begin{equation*}
    \Delta_{t_1} x_2 + \partial_{t_2} x_1 + \partial_{t_2} \Delta_{t_1} x_1 = - \eta \left[y_0 x_2 + \frac{3}{2} \zeta x_1^2 +  x_1 y_1 - 1\right] \, .
\end{equation*}
Here, we focus on understanding what happens at the edge of stability, i.e., when $\eta y_0 = 2$. We then know that $x_1 = (-1)^{t_1} a_1(t_2, t_3)$, and thus $\Delta_{t_1} x_1 = -2 (-1)^{t_1} a_1(t_2, t_3)$. We obtain 
\begin{equation}
    \label{eq:x2}
    \Delta_{t_1} x_2 +  2 x_2 = (-1)^{t_1} \partial_{t_2} a_1 - \eta \left[\frac{3}{2} \zeta a_1^2 +  (-1)^{t_1} a_1 y_1 - 1\right] \, .
\end{equation}
Consider the above equation as a difference equation for $x_2$ in $t_1$, with fixed $t_2$ and~$t_3$. The right-hand side depends on $t_1$ only through the $(-1)^{t_1}$ terms, as the terms $a_1$ and~$y_1$ were shown to be independent of $t_1$. This suggests considering the following lemma. 

\begin{lem}
    \label{lem:secular}
    Let $\alpha, \beta \in \R$. The solutions $\varphi = \varphi(t_1)$ of the difference equation
    \begin{equation*}
        \Delta_{t_1} \varphi + 2 \varphi = \alpha + \beta (-1)^{t_1} 
    \end{equation*}
    are 
    \begin{equation*}
        \varphi = \frac{\alpha}{2} + a(-1)^{t_1} - \beta t_1 (-1)^{t_1} \, , \qquad a \in \R \, .
    \end{equation*}
    In particular, there exists a bounded solution if and only if $\beta = 0$.  
\end{lem}

\begin{proof}
    This lemma is a particular case of Lemma~\ref{lem:secular-general} below.
\end{proof}

Applying this lemma to Eq.~\eqref{eq:x2}, we obtain that we must have $\partial_{t_2} a_1 = \eta y_1 a_1$ to avoid a secular term. The lemma also gives an expression for $x_2$, but it turns out that we do not need it in this section. This is typical of the method of multiple scales: we solve higher-order terms not always to include them in our approximation, but to obtain information about lower-order terms---here, the variations of $a_1$ in $t_2$.

We now turn to the $y$ equation \eqref{eq:y}. It gives
\begin{equation*}
    \Delta_{t_1} y_2 + \partial_{t_2} y_1 + \partial_{t_3} y_0 = - \eta \left[\frac{x_1^2}{2} - 1\right] \, .
\end{equation*}
As before, we consider this as a difference equation for $y_2$ in $t_1$: 
\begin{equation*}
    \Delta_{t_1} y_2  = - \partial_{t_2} y_1 - \partial_{t_3} y_0- \eta \left[\frac{x_1^2}{2} - 1\right] \, .
\end{equation*}
We distinguish whether we are at the edge of stability. 
\begin{itemize}
    \item If $\eta y_0 < 2$, then $x_1 \to 0$ as $t_1 \to \infty$, and all other quantities of the right-hand side are independent of $t_1$. To prevent a secular term, we must have 
    \begin{equation*}
        - \partial_{t_2} y_1 - \partial_{t_3} y_0 + \eta = 0 \, .
    \end{equation*}
    We consider this as an ODE for $y_1$ in $t_2$. As $y_0$ is independent of $t_2$, we have a secular term unless 
    \begin{equation*}
        \partial_{t_3} y_0 = \eta \, . 
    \end{equation*}
This finishes our derivation of item \ref{it:interior} of the result: in the relative interior of the stable region, the dynamics follow the gradient flow of $h$ constrained to the valley.
    \item If $\eta y_0 = 2$, then we have $x_1 = (-1)^{t_1} a_1(t_2, t_3)$ and $\partial_{t_3} y_0 = 0$. We obtain
    \begin{equation*}
        \Delta_{t_1} y_2  = - \partial_{t_2} y_1 - \eta \left[\frac{a_1^2}{2} - 1\right] \, .
    \end{equation*}
    The right-hand side is independent of $t_1$. To avoid a secular term, we must have 
    \begin{equation*}
        \partial_{t_2} y_1 = \eta \left(1 - \frac{a_1^2}{2}\right) \, .
    \end{equation*}
    This completes the derivation of item \ref{it:eos} of the result. 
\end{itemize}
\end{proof}

\begin{remark}
    \label{rmk:cf-failure}
    Without making the assumption $\eta y_0 \leq 2$, we would have considered a third case $\eta y_0 > 2$ in the derivations of order $\varepsilon^{1/2}$. In this case, $x_1$ develops a secular term (exponentially growing), unless $x_1 = 0$. Thus, to avoid a secular term, we must have the dichotomy: $\eta y_0 \leq 2$ or $x_1 = 0$. It is not excluded that we leave the stable valley, however this can be done only in a way where the perturbation $x_1$ of $x_0 = 0$ is exactly $0$. This situation seems non-generic; in our case the term $-x$ of $h$ maintains a small drift from $x = 0$ that prevents this situation. In Fig.~\ref{fig:minimal-example-4}, we observe that removing the term $-x$ from $h$ results in the iterates leaving the stable valley for a significant slow time $t_3$, thus breaking the central flow prediction. More generally, we conjecture that a perturbation in $x$ (transversal gradient of $h$, noise, numerical errors, etc.) is necessary for the central flow prediction to hold. 
\end{remark}

\begin{figure}[ht]
    \centering
        \includegraphics[width=\linewidth]{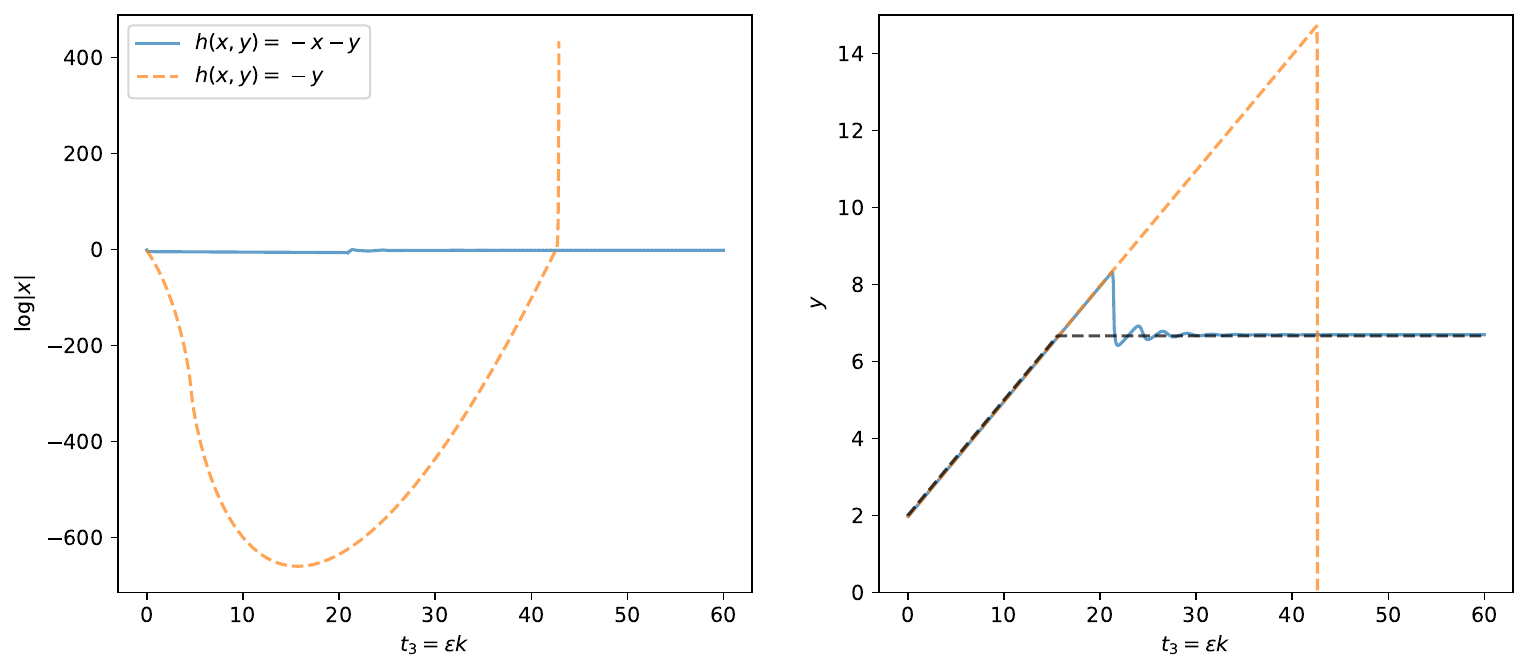}
    \caption{Iterates of gradient descent where $g$ is chosen as in Eqs.~\eqref{eq:minimal} and $h(x,y) = -x-y$ or $h(x,y)= -y$. Here, $\eta = 0.3$, $\zeta = 1$ and $\varepsilon = 0.03$. In the absence of the term $-x$ in $h$, $x$ becomes exponentially small in the first part of the dynamics, and thus it is possible to leave the stable valley, with the self-stabilization mechanism kicking in only after a significant slow time $t_3$. The self-stabilization mechanism is then so strong that the trajectory diverges. }
    \label{fig:minimal-example-4}
\end{figure}

\subsection{Energy variations in self-stabilization}
\label{sec:energy}

We have seen above that the self-stabilization dynamics preserve the energy 
\begin{equation*}
    E =E(t_3)= \frac{y_1^2}{2} + \frac{a_1^2}{4} - \frac{1}{2} \log a_1^2 \, ,
\end{equation*}
as a function of $t_2$. However, the energy still varies as a function $t_3$, as observed in Figs.~\ref{fig:minimal-example-2} and \ref{fig:minimal-example-3}. The method of multiple scales allows us to study these variations.

\begin{result}
    \label{result:energy-variations}
    Under the assumptions of Result~\ref{result:minimal}, at the edge of stability, we have 
    \begin{equation*}
        \partial_{t_3} E = \eta^2 (2 - 9 \zeta^2) \left\langle y_1^2 \right\rangle_{t_2} \, ,
    \end{equation*}
    where $\left\langle y_1^2 \right\rangle_{t_2} = \left\langle y_1^2(t_2, t_3) \right\rangle_{t_2}$ denotes the average of $y_1^2$ over one period in $t_2$.
\end{result}

Note that $\left\langle y_1^2 \right\rangle_{t_2}$ is a function of $E$ only, thus the above result is an autonomous ODE for $E$ in $t_3$. If $\zeta^2 > \frac{2}{9}$, the energy decreases while if $\zeta^2 < \frac{2}{9}$, the energy increases, see Fig.~\ref{fig:energy}. However, in the latter case, the energy does not diverge but saturates at a large value. In fact, $E$ becomes so large that the scales of $\varepsilon^{1/2}a_1, \varepsilon^{1/2}y_1$ are no longer of order $\varepsilon^{1/2}$, breaking the asymptotic expansion ansatz used to derive Result~\ref{result:energy-variations}. Our derivations are then no longer predictive of the actual dynamics and more refined computations would be required to study the saturation of the energy. 

\begin{figure}[ht]
    \centering
    \begin{subfigure}[c]{0.49\textwidth}
        \centering
        \includegraphics[width=\linewidth]{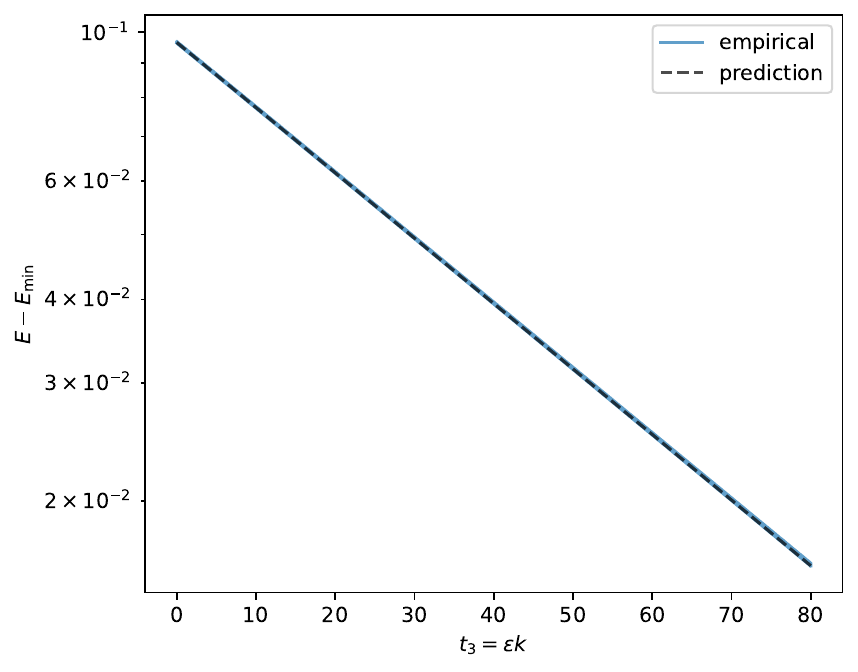}
        \caption{$\zeta = 0.5$}
    \end{subfigure}
    \hfill
    \begin{subfigure}[c]{0.49\textwidth}
        \centering
        \includegraphics[width=\linewidth]{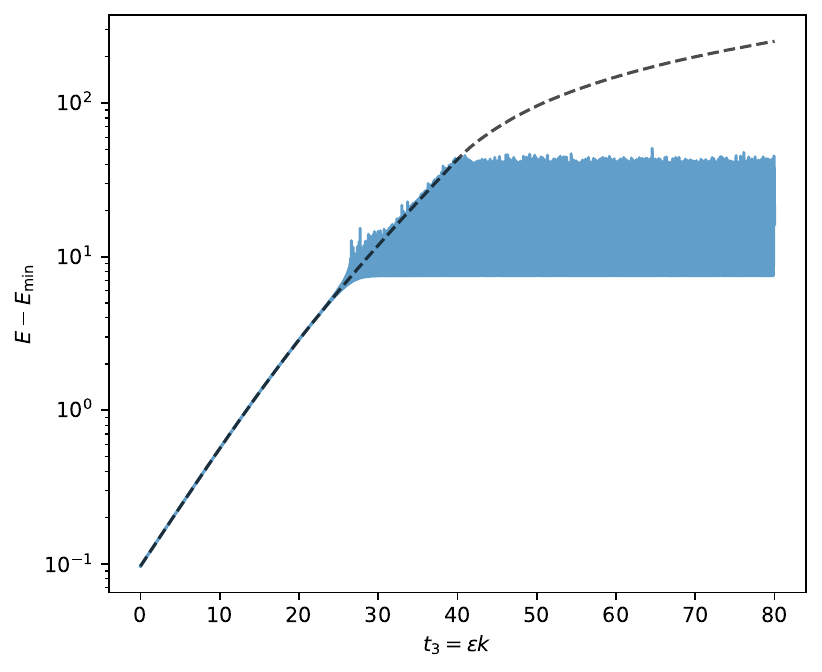}
        \caption{$\zeta = 0$}
    \end{subfigure}
    \caption{Comparison between the empirical variation of the energy and the theoretical prediction of Result~\ref{result:energy-variations}. This experiment is initialized at the edge of stability. Here, $\eta = 0.3$ and $\varepsilon = 10^{-5}$. More precisely, we define $y_1^{\text{emp}} = \varepsilon^{-1/2} (y - y_0)$ and $a_1^{\text{emp}} = \varepsilon^{-1/2} (-1)^{t_1}(x - x_0)$. This defines an empirical energy $E^{\text{emp}} = (y_1^{\text{emp}})^2/2 + (a_1^{\text{emp}})^2/4 - \log((a_1^{\text{emp}})^2)/2$, that we compare against $E^{\text{pred}}(t_3) = E^{\text{emp}}(0) + \eta^2 (2 - 9 \zeta^2) \int_{0}^{t_3} (y_1^{\text{emp}})^2 ds_3$. We denote $E_{\min} = \frac{1}{2}(1 - \log 2)$ the minimal possible energy.}
    \label{fig:energy}
\end{figure}

To summarize, the self-stabilization system~\eqref{eq:self-stab-minimal} and Result~\ref{result:energy-variations} should be understood as an analysis at energy levels of order $1$. This analysis is most relevant for practice since neural networks seem to systematically dissipate energy when only one eigenvalue is on the edge of stability, see Fig.~\ref{fig:nn} of this paper or Figs.~5, 6, 10, 14, 15 in~\cite{cohen2024understanding}. We leave it as an open problem to understand the structure of the loss of neural networks that causes this mechanism. 

\begin{figure}[ht]
    \centering
    \includegraphics[width=0.5\linewidth]{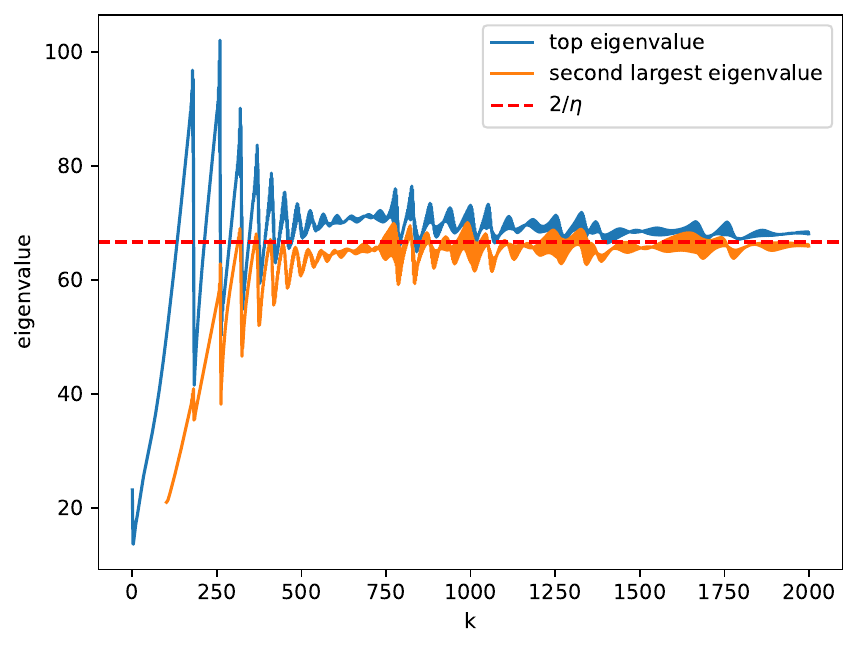}
    \caption{Evolution of the top eigenvalues of the Hessian along the training path of a neural network. The amplitude of the self-stabilization fluctuations decreases when only one eigenvalue is on the edge of stability. When a second eigenvalue reaches the edge of stability, the amplitude of the fluctuations stabilizes at a larger value, as expected from the analysis of Sec.~\ref{sec:general}. Here, we train a multilayer perceptron with two hidden layers of width 256 and GELU activation on a subset of CIFAR-10 with mean-squared error loss. The learning rate is $\eta = 0.03$.}
    \label{fig:nn}
\end{figure}

Analytically, studying the dependence of $E$ in $t_3$ requires pushing the method of multiple scales to the order $\varepsilon^{3/2}$. Indeed, as the reader might have noticed, the order $1$ gives the dependence of $z_0$ in $t_1$; the order $\varepsilon^{1/2}$ gives the dependencies of $z_1$ in $t_1$ and of $z_0$ in $t_2$; the order $\varepsilon$ gives the dependencies of $z_2$ in $t_1$, of $z_1$ in $t_2$, and of $z_0$ in $t_3$. Continuing this pattern, the dependence of $z_1$ (and thus of $E$) in $t_3$ is obtained at order~$\varepsilon^{3/2}$. This leads to rather tedious computations provided in Appendix \ref{sec:proof-energy-variations}. We have done these derivations only in the minimal example of this section and leave the more general case of Sec.~\ref{sec:codim-one} for future work.

\section{Valleys of codimension one}
\label{sec:codim-one}

We now return to the general setup of Sec.~\ref{sec:setting}. Here, we are interested in the fact that $Y$ can be multidimensional, thus $y_0$ can move at the edge of stability. However, in this section, we still assume that $X = \R$, so that $\nabla^2g(0,y)$ has exactly one non-zero eigenvalue $S(0,y) = \partial^2_{x} g(0,y)$. In Figure~\ref{fig:codim-one-proposal}, we provide an example of such gradient descent dynamics where the valley is of dimension $2$. 

We now state a general result for such valleys of codimension $1$. This result is an extension of Result \ref{result:conv-cf} and a generalization of Result~\ref{result:minimal}. 

\begin{result}
    \label{result:codim-one}
    Assume that $X$ is of dimension $1$. Further, assume that $(x_0, y_0)$ remains in the stable valley, i.e., $x_0 = 0$ and $\eta S(0,y_0) \leq 2$. Then $y_0 = y_0(t_3)$, and $y_1 = y_1(t_2, t_3)$. Further, we have the following:
\begin{enumerate}[label=(\roman*)]
    \item\label{it:codim-one-1} There exists $\sigma^2 = \sigma^2(t_3) \geq 0$ such that 
    \begin{align*}
        &\partial_{t_3} y_0 = - \eta \left[\nabla_y h(0, y_0) + \frac{\sigma^2}{2} \nabla_y S(0, y_0)\right] \, , 
    \end{align*}
    where for all $t_3$, $\sigma^2(t_3) = 0$ or $\eta S(0, y_0(t_3)) = 2$.
    \item\label{it:codim-one-2} If $\eta S(0, y_0) = 2$ (edge of stability), then $x_1 = (-1)^{t_1} a_1$ for some $a_1 = a_1(t_2, t_3)$. Further, denote $u_1 = \langle y_1, \nabla_y S(0, y_0) \rangle$. Then we have the self-stabilization system
    \begin{align}
        \label{eq:self-stab-codim-one}
        \begin{split}
        \partial_{t_2} a_1 &= \eta a_1 u_1 \, , \\
        \partial_{t_2} u_1 &= - \eta \left[\left\langle \nabla_y h(0, y_0), \nabla_y S(0, y_0) \right\rangle + \Vert \nabla_y S(0, y_0) \Vert^2 \frac{a_1^2}{2} \right] \, .
        \end{split}
    \end{align}
    Moreover, in that case, we have $\sigma^2 = \langle a_1^2 \rangle_{t_2}$, i.e.~$\sigma^2(t_3)$ is the average of $a_1^2(t_2, t_3)$ over a period in $t_2$.
    \end{enumerate}
    \end{result}

    The derivation of this result is provided in Appendix~\ref{sec:derivation-codim-one}. Note that in the self-stabilization system~\eqref{eq:self-stab-codim-one}, $\left\langle \nabla_y h(0, y_0), \nabla_y S(0, y_0) \right\rangle$ and $\Vert \nabla_y S(0, y_0) \Vert^2$ are independent of $t_2$. Moreover, in a situation where the dynamics stay at the edge of stability, we must have $\left\langle \nabla_y h(0, y_0), \nabla_y S(0, y_0) \right\rangle < 0$: the gradient flow on $h$ pushes towards the stability boundary. The self-stabilization system~\eqref{eq:self-stab-codim-one} is thus the same as the one of \eqref{eq:self-stab-minimal}, up to a rescaling. A self-stabilization loop appears, preserving an energy $E = E(t_3)$. For fixed $t_3$, $(a_1, u_1)$ is periodic in $t_2$, thus the average $\langle a_1^2 \rangle_{t_2}$ over $t_2$ is well-defined. 

    The derivation of the central flow in \cite{cohen2024understanding} already interpreted $\sigma^2 = \E x^2 $, where $\E$ denotes ``local time-averages''. Here, thanks to the timescale decoupling, we can be more precise: $\E$ corresponds to averaging over $t_2$, for fixed $t_3$.

\section{Valleys of arbitrary codimension}
\label{sec:general}

This section tackles the case where the dimension of $X$ is arbitrary. 

\begin{result}
    \label{result:general}
    Assume that $(x_0, y_0)$ remains in the stable valley, i.e., $x_0 = 0$ and $\eta S(0,y_0) \preceq 2 \Id_X$. Then $y_0 = y_0(t_3)$, and $y_1 = y_1(t_2, t_3)$. Further, we have the following:
    \begin{enumerate}[label=(\roman*)]
        \item\label{it:general-1} There exists $\Sigma = \Sigma(t_3) \in \Sym_+ X$ such that 
        \begin{equation*}
            \partial_{t_3} y_0 = - \eta \left[\nabla_y h(0, y_0) + \frac{1}{2} D_y S(0, y_0)^*[\Sigma]\right] \, ,
        \end{equation*}
        where for all $t_3$, $\langle \Sigma(t_3), 2 \Id_X - \eta S(0,y_0)\rangle = 0$. 
        \item Let $C = C(t_3) = \ker(2 \Id_X - \eta S(0, y_0)) \subset X$ denote the critical subspace. Recall that $\iota_C : C \to X$ denotes the canonical inclusion, and $P_C : X \to C$ the orthogonal projection onto $C$. We have $P_C x_1 = (-1)^{t_1} a_1$ for some $a_1 = a_1(t_2, t_3) \in C$ and $P_{C^\perp}x_1 \xrightarrow[t_1 \to \infty]{} 0$. Further, denote 
        \begin{align*}
            &\Psi = \Psi(t_3): Y \to \Sym C \\
            &\hspace{5.5em} \delta y \mapsto \Psi[\delta y] = P_C D_y S(0, y_0)[\delta y] \iota_C \, ,
        \end{align*}
        and $U_1 = \Psi[y_1] \in \Sym C$. Then we have the self-stabilization system
        \begin{align}
            \label{eq:self-stab-general}
            \begin{split}
            \partial_{t_2} a_1 &= \eta U_1 a_1 \, , \\
            \partial_{t_2} U_1 &= - \eta \left[ \Psi[\nabla_y h(0, y_0)] + \frac{1}{2} \Psi\Psi^*[a_1 a_1^\top ]\right] \, 
            \end{split}
        \end{align}
        Moreover, provided that the average $\langle a_1 a_1^\top \rangle_{t_2} = \lim_{T_2 \to \infty} \frac{1}{T_2} \int_0^{T_2} a_1(t_2, t_3)a_1(t_2, t_3)^\top \diff t_2$ is well-defined, we have $\Sigma = \iota_C \langle a_1 a_1^\top \rangle_{t_2} P_C$. 
    \end{enumerate}
\end{result}

The derivation of this result is provided in Appendix~\ref{sec:derivation-general}. The derivation of the central flow in \cite{cohen2024understanding} already interpreted $\Sigma = \E aa^\top $, where $\E$ denotes ``local time-averages''. Here, thanks to the timescale decoupling, we can be more precise: $\E$ corresponds to averaging over $t_2$, for fixed $t_3$.

Result \ref{result:general} is a generalization of Result~\ref{result:codim-one}. When only one eigenvalue of $S(0,y_0)$ has reached the edge of stability, we recover the self-stabilization system~\eqref{eq:self-stab-codim-one}. Indeed, then $C$ is of dimension $1$, thus $a_1 \in C$ and $U_1 \in \Sym C$ are scalars. Moreover, $\Psi\Psi^*: \Sym C \to \Sym C$ is a linear operator on a one-dimensional space, thus it is a scalar multiplication. Thus in that case, the self-stabilization system~\eqref{eq:self-stab-general} reduces to the self-stabilization system~\eqref{eq:self-stab-codim-one}. We have a periodic behavior in $t_2$. 

When more eigenvalues enter the edge of stability, the behavior of the self-stabilization system~\eqref{eq:self-stab-general} is much richer, see Fig.~\ref{fig:general}. 

The self-stabilization system~\eqref{eq:self-stab-codim-one} admits the fixed point $u_1 = 0, a_1^2 = -\frac{2 \langle \nabla_y h, \nabla_y S \rangle}{\Vert \nabla_y S \Vert^2}$. Meanwhile, finding a fixed point for the self-stabilization system~\eqref{eq:self-stab-general} would require solving the equation $\Psi\Psi^*[a_1 a_1^\top] = - 2 \Psi[\nabla_y h(0, y_0)]$. As the right hand side is in $\im \Psi = \im \Psi\Psi^*$, there exists $\widetilde{\Sigma} \in \Sym C$ such that $\Psi\Psi^*[\widetilde{\Sigma}] = -2 \Psi[\nabla_y h(0, y_0)]$. However, it is non-generic that a rank-one solution $\widetilde{\Sigma}$ exists. Thus, generically, there is no fixed point, and to avoid a secular term, $a_1 a_1^\top$ must evolve persistently so that its average value is $\widetilde{\Sigma}$.

This explains the observed behavior in Fig.~\ref{fig:general}. When only one eigenvalue is on the edge of stability, the self-stabilization system is observed to lose energy (on the $t_3$ timescale), converging to its fixed point. However, when a second eigenvalue enters the edge of stability, there is a significant increase in the edge of stability fluctuations, and these fluctuations are persistent. This phenomenology, here illustrated on a toy example, has been observed systematically in deep learning experiments, see Fig.~\ref{fig:nn} of this paper or Figs.~5, 6, 10, 14, 15 in~\cite{cohen2024understanding}.

To finish, we now discuss the implications of Result~\ref{result:general} for the loss $f(x,y)$, as it is the most commonly observed quantity in practice. 
    
    \begin{coro}
        \label{coro:general}
        Under the assumptions of Result~\ref{result:general}, we have
        \begin{equation*}
            f(x,y) = \min g + \varepsilon f_2 + o(\varepsilon) 
        \end{equation*}
        where 
        \begin{equation*}
            f_2 \xrightarrow[t_1 \to \infty]{} h(0, y_0) + \frac{\Vert a_1 \Vert^2}{\eta} \, .
        \end{equation*}
    \end{coro}

    The derivation of this corollary is provided in Appendix~\ref{sec:derivation-coro-general}. The dominant non-constant term of the loss is of order $\varepsilon$. In the relative interior of the stable valley, this term depends only on $t_3$ (up to a potential initial transient term in~$t_1$), through the component $h$ of the loss. At the edge of stability, an additional increase ${\Vert a_1 \Vert^2}/{\eta}$ appears, depending on both $t_2$ and $t_3$. In $t_2$, the self-stabilization fluctuations induce spikes in the loss. Averaging over these, and assuming that the average $\langle a_1 a_1^\top \rangle_{t_2}$ is well-defined, we obtain
    \begin{equation*}
        \langle f_2 \rangle_{t_2} \xrightarrow[t_1 \to \infty]{} h(0, y_0) + \frac{\Tr \Sigma}{\eta} \, ,
    \end{equation*}
    consistently with the result of \cite[Sec.~3.2.3]{cohen2024understanding}. These results are illustrated in Fig.~\ref{fig:loss}. 

    \begin{figure}[ht]
    \centering
        \includegraphics[width=\linewidth]{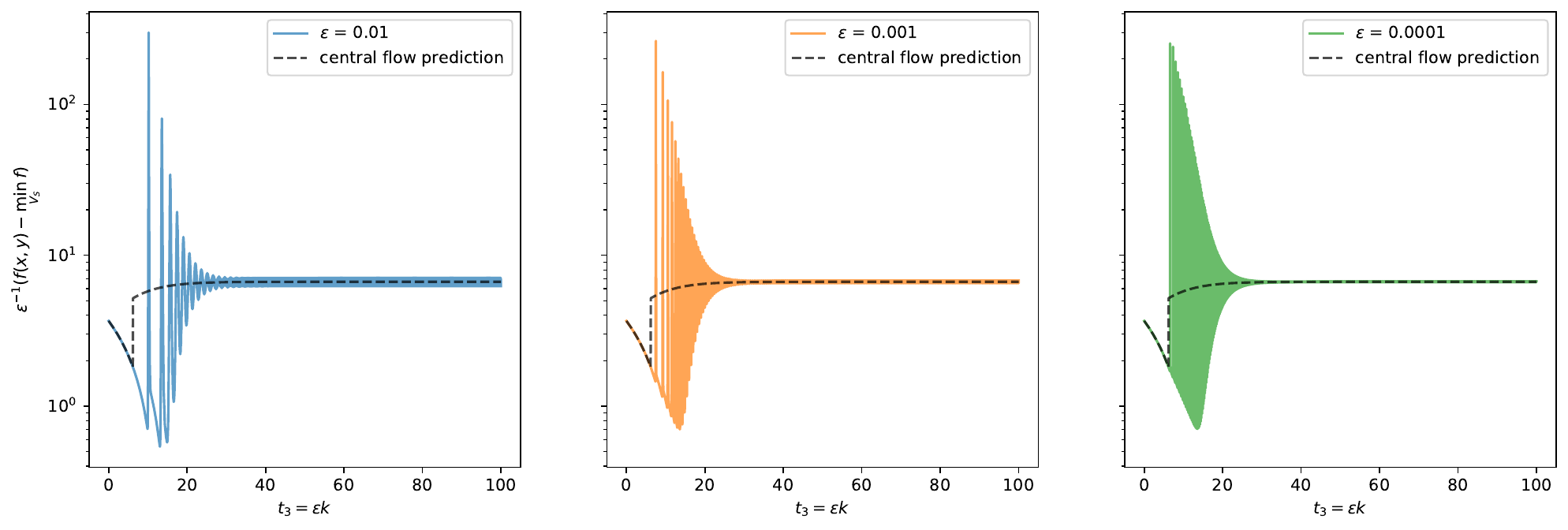}
    \caption{Rescaled loss $\varepsilon^{-1}\left(f(x(k),y(k)) - \min_{(x,y) \in V_{\mathrm{S}}} f(x,y)\right)$ along the iterates of gradient descent and its central flow prediction $h(0,y_0) - \min_{(x,y) \in V_{\mathrm{S}}} h(x,y)\min + \sigma^2/\eta$, in the same experiment as in Fig.~\ref{fig:codim-one-proposal}. The loss progress made on $h$ in the valley and the loss increase due to oscillations are both of the same order $\varepsilon$.}
    \label{fig:loss}
\end{figure}

\section{Related work}
\label{sec:related}

\paragraph{Sharp valley structure of the loss.} Arguments for the sharp valley structure of the loss landscape are widespread in the literature. To give a few examples, this picture is consistent with plots of the loss along the sharpest direction of the Hessian around an iterate \cite{jastrzkebski2018relation}, plots of the loss along the segment between two consecutive iterates \cite{xing2018walk}, and the anti-alignment of the gradients of consecutive iterates \cite{xing2018walk}. Along the training path, the loss Hessian exhibits an approximate low-rank structure \cite{sagun2016eigenvalues, sagun2017empirical, ghorbani2019investigation, pmlr-v97-papyan19a, papyan2020traces,ben2024high}. The gradients are mostly aligned with the sharpest directions of the Hessian \cite{gur2018gradient,li2020hessian}, however the component of the gradients in these dominant directions does not contribute to any loss decrease \cite{jastrzkebski2018relation,song2025does}, but to a sharpness reduction effect \cite{so2026mini}. In a toy problem, the valley structure has been mathematically analyzed and the sharpest directions are associated with uncertain features of the problem \cite{wen2024understanding}. Finally, given this accumulation of empirical evidence, the sharp valley structure has already been used as a working hypothesis in some works \cite{liu2025focus, liu2025neural}. 

\paragraph{Analyses of the edge of stability phenomenon.} After the initial empirical observations of the edge of stability phenomenon \cite{jastrzkebski2018relation,jastrzebski2020break,cohen2021gradient}, a large body of work has been devoted to explaining it \cite{ahn2022understanding,arora2022understanding,wang2022analyzing,damian2022self,zhu2022understanding,ma2022beyond,pmlr-v202-agarwala23b,chen2023beyond,kreisler2023gradient,song2023trajectory,wu2023implicit,wu2024large,cai2024large,wang2025good,kalra2025universal,ghosh2025learning,yoo2025understanding,liu2026minimalist, wu2026large,litman2026origin}. In contrast, our results \emph{assume} that the leading-term of dynamics remains in the stable region, and study the consequences of this assumption. To the best of our knowledge, no previous work had provided additional justification for the central flow approximation~\cite{cohen2024understanding}. 


From a modeling perspective, our work bears similarities with \cite{macdonald2026convergence,macdonald2026dynamics} that study the dynamics of gradient descent around a valley of minimizers. They show that gradient descent implicitly performs a Riemannian gradient descent on the sharpness in the valley of minimizers, and describe the bifurcating dynamical system in the orthogonal direction. In their ``subcritical regime'', their analysis corresponds to the case $h=0$ in our work. However, in the absence of $\nabla h$ pushing towards the stability boundary, the edge of stability phenomenon is transient.

\paragraph{Comparison with other approximations.} Several works \cite{regis2026rod,marion2026edge,de2026gradient} subsequent to the central flow \cite{cohen2024understanding} have proposed other averaged approximations of gradient descent, see Fig.~\ref{fig:comparison} for a comparison. In addition to the slow dynamics in $t_3$, these approximations attempt at capturing the self-stabilization dynamics in $t_2$. In fact, it is even explicit in the derivations of \cite{regis2026rod,marion2026edge} that they average only two subsequent iterates, thus over the timescale $t_1$ only. All of these approximations report similar or slightly better accuracy than the central flow in approximating gradient descent. While qualitatively, these approximations correctly generate self-stabilization loops, these loops are often off in magnitude or in phase, leading to only minor quantitative improvement. 

Overall, in the presence of the several timescales $t_1, t_2, t_3$, effective models can average at different levels, leading to different tradeoffs between accuracy and conceptual simplicity. The central flow, with its constrained gradient flow formulation, is the simplest and coarsest version, averaging over $t_1$ and $t_2$. The other approximations \cite{regis2026rod,marion2026edge,de2026gradient}, averaging over $t_1$ only, contain more information on the self-stabilization dynamics but have a more complex formulation. This paper advocates for seeking \emph{expansions} of the dynamics, where each individual term is simple and interpretable, while their combination can have high accuracy. 

\begin{figure}[ht]
    \centering
        \includegraphics[width=0.5\linewidth]{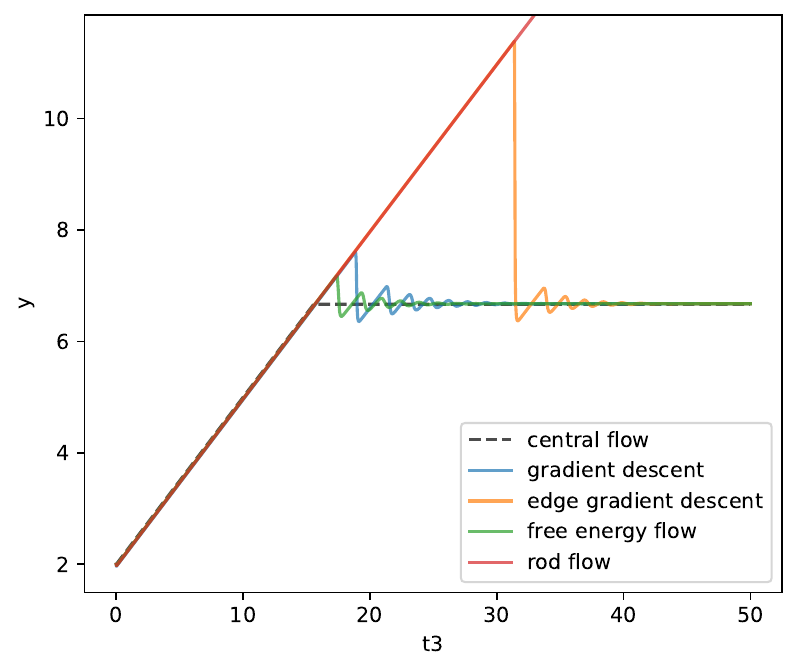}
    \caption{Comparison of the central flow \cite{cohen2024understanding}, the rod flow \cite{regis2026rod}, the edge gradient descent \cite{marion2026edge} and the flow of \cite{de2026gradient} (labeled ``free energy flow'' here) in the setting of Fig.~\ref{fig:minimal-example}. The fact that the rod flow may fail to detect the start of the edge of stability phase had already been pointed out in \cite[Fig.~6]{marion2026edge}.}
    \label{fig:comparison}
\end{figure}

\section{Conclusion}

The contributions of this paper are both conceptual and analytical. Conceptually, we clarify that the central flow and the self-stabilization mechanism are related to a sharp valley structure of the loss landscape. However, understanding why such a loss landscape structure appears when training deep learning models is still largely an open problem (see \cite{wen2024understanding} for a first approach).

Analytically, we provide a framework for studying the dynamics of gradient descent on the edge of stability, based on the method of multiple scales. This framework could be improved by formalizing the derivations through more rigorous techniques of singular perturbation theory. Computing the effect of a curved valley $V$ is also an interesting direction for future work. The works of \cite{macdonald2026convergence,macdonald2026dynamics} provide interesting first elements in these directions. 

The strength of the framework stems from its systematic nature. As a consequence, we plan to use it to analyze inertial and adaptive algorithms. We hope that understanding these complex dynamics across many optimizers will help to design new methods.

\section*{Acknowledgements}

The author thanks Victor Baillet, Léo Dana and Loucas Pillaud-Vivien for fruitful discussions. He also thanks the authors of \cite{cohen2024understanding} for making their code publicly available; in particular, this code was used to generate Figure~\ref{fig:nn}.

AI-based tools were used for line completion in the text editor, for generating the plotting scripts that produce the figures from the data, and for proofreading. They were not used to generate the scientific content of the paper, to draft its prose beyond phrase-level completion, or to write the code implementing the experiments. The author has verified all AI-assisted output and takes full responsibility for the content of the paper.

The author acknowledges support from the ANR and the Ministère de l’Enseignement Supérieur et de la Recherche.

\addcontentsline{toc}{section}{References}
\bibliographystyle{plain}
\bibliography{bibliography}

\newpage
\appendix

\vspace*{1cm}

\begin{center}
	\textsc{\Huge Appendix}	
\end{center}

\addtocontents{toc}{\protect\setcounter{tocdepth}{-1}}

\section{Derivation of Result~\ref{result:energy-variations}}
\label{sec:proof-energy-variations}

    We compute the terms of order $\varepsilon^{3/2}$ in the equation $\Delta_k z = - \eta \nabla f(z)$. For the right-hand side, we have
    \begin{align*}
    -\eta \nabla f(z) &= -\eta \nabla g(z) - \eta \varepsilon \nabla h(z) \nonumber\\
    &= -\eta \left[\nabla g\left(z_0 + \varepsilon^{1/2} z_1 + \varepsilon z_2 + \varepsilon^{3/2} z_3 + o(\varepsilon^{3/2})\right) + \varepsilon \nabla h\left(z_0 + \varepsilon^{1/2} z_1 + o(\varepsilon^{1/2})\right) \right]  \nonumber\\
    &= -\eta \bigg[\nabla g(z_0) + \nabla^2 g(z_0) \left(\varepsilon^{1/2} z_1 + \varepsilon z_2 + \varepsilon^{3/2} z_3\right) \\
    &\qquad\quad+ \frac{1}{2} \nabla^3 g(z_0) \left[\varepsilon^{1/2} z_1 + \varepsilon z_2 , \varepsilon^{1/2} z_1 + \varepsilon z_2 \right] \\
    &\qquad\quad+ \frac{1}{6} \nabla^4 g(z_0) \left[\varepsilon^{1/2} z_1, \varepsilon^{1/2} z_1, \varepsilon^{1/2} z_1 \right] \nonumber\\
    &\qquad\quad+ \varepsilon \nabla h (z_0) + \varepsilon \nabla^2h(z_0) \left(\varepsilon^{1/2}z_1\right) + o (\varepsilon^{3/2}) 
    \bigg] \nonumber\\
    &= -\eta \nabla g(z_0) -  \varepsilon^{1/2} \eta \nabla^2 g(z_0) z_1 - \varepsilon\eta \left[\nabla^2 g(z_0) z_2 + \frac{1}{2} \nabla^3 g(z_0) [z_1, z_1] + \nabla h(z_0)\right] \nonumber\\
    &\qquad
    - \varepsilon^{3/2}\eta \left[\nabla^2 g(z_0) z_3 + \nabla^3 g(z_0)[z_1, z_2] + \frac{1}{6} \nabla^4 g(z_0)[z_1, z_1, z_1] + \nabla^2h(z_0)z_1  \right]\\
    &\qquad+ o(\varepsilon^{3/2}) \, .
\end{align*}
Using the expressions for $g$ and $h$ of Eqs.~\eqref{eq:minimal} and the assumption $x_0 = 0$, we obtain 
\begin{align*}
    \begin{split}
    -\eta \partial_x f(x,y) &= - \varepsilon^{1/2} \eta y_0 x_1 - \varepsilon \eta \left[y_0 x_2 + \frac{3}{2} \zeta x_1^2 +  x_1 y_1 - 1\right] \\
    &\qquad- \varepsilon^{3/2} \eta \left[y_0 x_3 + 3 \zeta x_1 x_2 + x_1 y_2 + x_2 y_1\right]+ o(\varepsilon^{3/2}) \, , \\
    -\eta \partial_y f(x,y) &= - \varepsilon \eta \left[\frac{x_1^2}{2} - 1\right] - \varepsilon^{3/2} \eta x_1x_2 + o(\varepsilon^{3/2}) \, .
    \end{split}
\end{align*}
For the left-hand side, we first decompose 
\begin{align*}
    &z(t_1 , t_2 + \varepsilon^{1/2}, t_3 + \varepsilon) - z(t_1, t_2, t_3) \\
    &\qquad= \varepsilon^{1/2} \partial_{t_2} z + \varepsilon \partial_{t_3} z + \frac{\varepsilon}{2} \partial_{t_2}^2 z + \varepsilon^{3/2} \partial_{t_2} \partial_{t_3} z + \frac{\varepsilon^{3/2}}{6} \partial_{t_2}^3 z + o(\varepsilon^{3/2}) \, .
\end{align*}
Using a similar expansion for $(\Delta_{t_1} z)(t_1, t_2 + \varepsilon^{1/2}, t_3 + \varepsilon)$ and substituting in Eq.~\eqref{eq:aux-4}, we obtain
\begin{align*}
    \Delta_k z &= \Delta_{t_1} z + \varepsilon^{1/2} \partial_{t_2} z + \varepsilon \partial_{t_3} z + \frac{\varepsilon}{2} \partial_{t_2}^2 z + \varepsilon^{3/2} \partial_{t_2} \partial_{t_3} z + \frac{\varepsilon^{3/2}}{6} \partial_{t_2}^3 z \\
    &\qquad+ \varepsilon^{1/2} \partial_{t_2} \Delta_{t_1} z + \varepsilon \partial_{t_3} \Delta_{t_1}z + \frac{\varepsilon}{2} \partial_{t_2}^2 \Delta_{t_1}z + \varepsilon^{3/2} \partial_{t_2} \partial_{t_3} \Delta_{t_1}z + \frac{\varepsilon^{3/2}}{6} \partial_{t_2}^3 \Delta_{t_1}z +o(\varepsilon^{3/2}) \\
    &= \Delta_{t_1} z + \varepsilon^{1/2} \left[\partial_{t_2} z + \partial_{t_2} \Delta_{t_1} z\right] + \varepsilon \left[\partial_{t_3} z + \frac{1}{2} \partial_{t_2}^2 z + \partial_{t_3} \Delta_{t_1}z + \frac{1}{2} \partial_{t_2}^2 \Delta_{t_1}z\right] \\
    &\qquad +\varepsilon^{3/2} \left[ \partial_{t_2} \partial_{t_3} z + \frac{1}{6} \partial_{t_2}^3 z +  \partial_{t_2} \partial_{t_3} \Delta_{t_1}z + \frac{1}{6} \partial_{t_2}^3 \Delta_{t_1}z\right]+ o(\varepsilon^{3/2}) \, . 
\end{align*}
We combine the above decomposition with the expansion of $z$. We obtain 
\begin{align*}
    \begin{split}
    \Delta_k z &= \Delta_{t_1} z_0 + \varepsilon^{1/2} \left[\Delta_{t_1} z_1 + \partial_{t_2} z_0 +  \partial_{t_2} \Delta_{t_1}z_0 \right] \\
    &\quad+ \varepsilon \left[\Delta_{t_1} z_2 + \partial_{t_2} z_1 + \partial_{t_2} \Delta_{t_1} z_1 + \partial_{t_3} z_0 + \frac{1}{2} \partial_{t_2}^2 z_0  + \partial_{t_3} \Delta_{t_1} z_0 + \frac{1}{2} \partial_{t_2}^2 \Delta_{t_1} z_0\right] \\
    &\quad+ \varepsilon^{3/2} \bigg[ \Delta_{t_1}z_3 + \partial_{t_2} z_2 + \partial_{t_2} \Delta_{t_1} z_2 + \partial_{t_3} z_1 + \frac{1}{2} \partial_{t_2}^2 z_1 + \partial_{t_3} \Delta_{t_1}z_1 + \frac{1}{2} \partial_{t_2}^2 \Delta_{t_1}z_1 \\
    &\hspace{5em}+  \partial_{t_2} \partial_{t_3} z_0 + \frac{1}{6} \partial_{t_2}^3 z_0 +  \partial_{t_2} \partial_{t_3} \Delta_{t_1}z_0 + \frac{1}{6} \partial_{t_2}^3 \Delta_{t_1}z_0\bigg]+ o(\varepsilon^{3/2}) \, .
    \end{split}
\end{align*}
We are now ready to identify the terms of order $\varepsilon^{3/2}$. From Eq.~\eqref{eq:x}, we obtain
\begin{align*}
    &\Delta_{t_1} x_3 + \partial_{t_2} x_2 + \partial_{t_2} \Delta_{t_1} x_2 + \partial_{t_3} x_1 + \frac{1}{2} \partial_{t_2}^2 x_1 + \partial_{t_3} \Delta_{t_1}x_1 + \frac{1}{2} \partial_{t_2}^2 \Delta_{t_1}x_1 \\
    &\qquad= - \eta \left[y_0 x_3 + 3 \zeta x_1 x_2 + x_1 y_2 + x_2 y_1\right] \, .
\end{align*}
Here, we are interested in the edge of stability phase where $\eta y_0 = 2$, $x_1 = (-1)^{t_1}a_1(t_2, t_3)$. Moreover, applying Lemma~\ref{lem:secular} to Eq.~\eqref{eq:x2}, we have $x_2 = -\frac{\eta}{2} \left[\frac{3}{2}\zeta a_1^2-1\right]+(-1)^{t_1}a_2$ for some $a_2 = a_2(t_2, t_3)$. This gives
\begin{align*}
    \Delta_{t_1} x_3 + 2x_3 &= \frac{3\eta}{4} \zeta \partial_{t_2} \left(a_1^2\right) - (-1)^{t_1} \partial_{t_2} a_2 + 2(-1)^{t_1} \partial_{t_2} a_2  - (-1)^{t_1} \partial_{t_3} a_1 \\
    &\qquad- \frac{1}{2} (-1)^{t_1} \partial_{t_2}^2 a_1 + 2 (-1)^{t_1} \partial_{t_3} a_1 + (-1)^{t_1} \partial_{t_2}^2 a_1 \\
    &\qquad- \eta \bigg[3\zeta (-1)^{t_1} a_1 \left(-\frac{\eta}{2} \left[\frac{3}{2}\zeta a_1^2 - 1\right] + (-1)^{t_1}a_2\right) + (-1)^{t_1} a_1 y_2 \\
    &\qquad\qquad + \left(-\frac{\eta}{2} \left[\frac{3}{2}\zeta a_1^2 - 1\right] + (-1)^{t_1}a_2\right)y_1\bigg] \\
    &= \frac{3\eta}{4} \zeta \partial_{t_2} \left(a_1^2\right) + (-1)^{t_1} \partial_{t_2} a_2  + (-1)^{t_1} \partial_{t_3} a_1 + \frac{1}{2} (-1)^{t_1} \partial_{t_2}^2 a_1  \\
    &\qquad- \eta \bigg[3\zeta (-1)^{t_1} a_1 \left(-\frac{\eta}{2} \left[\frac{3}{2}\zeta a_1^2 - 1\right] + (-1)^{t_1}a_2\right) + (-1)^{t_1} a_1 y_2 \\
    &\qquad\qquad + \left(-\frac{\eta}{2} \left[\frac{3}{2}\zeta a_1^2 - 1\right] + (-1)^{t_1}a_2\right)y_1\bigg] \, . 
\end{align*}
The right hand side depends on $t_1$ only through the $(-1)^{t_1}$ terms. By Lemma~\ref{lem:secular}, to avoid a secular term, we must have
\begin{align}
    \label{eq:aux-a}
    \begin{split}
    \partial_{t_2} a_2 + \partial_{t_3} a_1 + \frac{1}{2} \partial_{t_2}^2 a_1 &= \eta \left(-3\zeta a_1 \frac{\eta}{2} \left[\frac{3}{2}\zeta a_1^2 - 1\right] + a_1 y_2 + a_2 y_1 \right) \\
    &= -\eta^2 \frac{3}{2} \zeta a_1 \left[\frac{3}{2}\zeta a_1^2 - 1\right] + \eta \left(a_1 y_2 + a_2 y_1 \right) \, .
    \end{split}
\end{align}
We now identify terms of order $\varepsilon^{3/2}$ in the $y$ equation~\eqref{eq:y}. Recall that $y_0 = 2/\eta$, $y_1 = y_1(t_2, t_3)$, $y_2 = y_2(t_2, t_3)$, leading to many simplifications. We obtain
\begin{align*}
    \Delta_{t_1} y_3 + \partial_{t_2} y_2 + \partial_{t_3} y_1 + \frac{1}{2} \partial_{t_2}^2 y_1 = - \eta x_1 x_2 = - \eta (-1)^{t_1} a_1 \left(-\frac{\eta}{2} \left[\frac{3}{2}\zeta a_1^2 - 1\right] + (-1)^{t_1}a_2\right) \, .
\end{align*}
We consider this as a difference equation for $y_3$ in $t_1$. \begin{equation*}
    \Delta_{t_1} y_3 = - \partial_{t_2} y_2 - \partial_{t_3} y_1 - \frac{1}{2} \partial_{t_2}^2 y_1 - \eta (-1)^{t_1} a_1 \left(-\frac{\eta}{2} \left[\frac{3}{2}\zeta a_1^2 - 1\right] + (-1)^{t_1}a_2\right) \, .
\end{equation*}
The right hand side depends on $t_1$ only through the $(-1)^{t_1}$ terms. To avoid a secular term, we must have 
\begin{align}
    \label{eq:aux-y}
    \partial_{t_2} y_2 + \partial_{t_3} y_1 + \frac{1}{2} \partial_{t_2}^2 y_1 &= - \eta a_1 a_2 \, .
\end{align}
Recall that we seek to compute 
\begin{equation*}
    \partial_{t_3} E = y_1 \partial_{t_3} y_1 + \left(\frac{a_1}{2}-\frac{1}{a_1}\right) \partial_{t_3} a_1 \, .
\end{equation*}
This suggests considering the quantity $F = y_1 y_2 + \left(\frac{a_1}{2}-\frac{1}{a_1}\right) a_2$. Indeed, using Eqs.~\eqref{eq:aux-a} and \eqref{eq:aux-y}, we have
\begin{align*}
    \partial_{t_2}F &= (\partial_{t_2} y_1) y_2 + y_1 (\partial_{t_2} y_2) + \left(\frac{1}{2} + \frac{1}{a_1^2}\right) (\partial_{t_2} a_1) a_2 + \left(\frac{a_1}{2}-\frac{1}{a_1}\right) (\partial_{t_2} a_2) \\
    &= \eta \left(1 - \frac{a_1^2}{2}\right) y_2 - y_1 \left(\partial_{t_3} y_1 + \frac{1}{2} \partial_{t_2}^2 y_1 + \eta a_1 a_2\right) + \left(\frac{1}{2} + \frac{1}{a_1^2}\right) (\eta a_1 y_1) a_2 \\
    &\qquad+ \left(\frac{a_1}{2}-\frac{1}{a_1}\right) \left(-\partial_{t_3} a_1 - \frac{1}{2} \partial_{t_2}^2 a_1 - \eta^2 \frac{3}{2} \zeta a_1 \left[\frac{3}{2}\zeta a_1^2 - 1\right] + \eta (a_1 y_2+ a_2 y_1)\right) \\
    &= -\partial_{t_3} E - \frac{y_1}{2} \partial_{t_2}^2 y_1 - \left(\frac{a_1}{2}-\frac{1}{a_1}\right) \frac{1}{2}\partial_{t_2}^2 a_1 - \eta^2 \frac{3}{2} \zeta \left(\frac{a_1^2}{2}-1\right)\left(\frac{3}{2}\zeta a_1^2 - 1\right) \, .
\end{align*}
The right-hand side depends on $t_2$ only through the functions $a_1$ and $y_1$, which are periodic in $t_2$. We then use the following lemma.
            
            \begin{lem}
                \label{lem:periodic}
                Let $t \mapsto \varphi(t)$ be a periodic real function. Then $\varphi(t)$ has a bounded primitive if and only if the average of $\varphi$ is $0$.
            \end{lem}To avoid a secular term in $F$ (which would induce a secular term in $y_2$ or $a_2$), we must have 
\begin{align*}
    \partial_{t_3} E &= \left\langle -\frac{y_1}{2} \partial_{t_2}^2 y_1 + \left(\frac{1}{2a_1} - \frac{a_1}{4}  \right)\partial_{t_2}^2a_1- \eta^2 \frac{3}{2} \zeta \left(\frac{a_1^2}{2}-1\right)\left(\frac{3}{2}\zeta a_1^2 - 1\right)\right\rangle_{t_2} \, .
\end{align*}
We now simplify this expression. We compute 
\begin{align*}
    \partial_{t_2}^2 y_1 &= \eta \partial_{t_2} \left(1 - \frac{a_1^2}{2}\right) = - \eta a_1 \partial_{t_2} a_1 = - \eta^2 a_1^2 y_1 \, , \\
    \partial_{t_2}^2 a_1 &= \eta \partial_{t_2} (a_1 y_1) = \eta \left[(\partial_{t_2}a_1) y_1 + a_1 (\partial_{t_2} y_1)\right] = \eta^2 \left[a_1 y_1^2 + a_1 \left(1-\frac{a_1^2}{2}\right)\right] \, .
\end{align*}
This gives
\begin{align*}
    \partial_{t_3} E &= \eta^2 \left\langle \frac{a_1^2}{2} y_1^2 + \left(\frac{1}{2a_1} - \frac{a_1}{4}\right) \left[a_1 y_1^2 + a_1 - \frac{a_1^3}{2}\right] - \frac{3}{2} \zeta \left(\frac{a_1^2}{2}-1\right)\left(\frac{3}{2}\zeta a_1^2 - 1\right)\right\rangle_{t_2} \\
    &= \eta^2 \left\langle \frac{a_1^2}{2} y_1^2 + \frac{y_1^2}{2} + \frac{1}{2} - \frac{a_1^2}{4} -  \frac{a_1^2}{4} y_1^2 - \frac{a_1^2}{4} + \frac{a_1^4}{8} - \frac{9}{8} \zeta^2 a_1^4 + \frac{3}{4} \zeta a_1^2 + \frac{9}{4} \zeta^2 a_1^2 - \frac{3}{2} \zeta\right\rangle_{t_2} \\
    &= \eta^2 \left( \frac{\langle a_1^2 y_1^2 \rangle_{t_2}}{4} + \frac{\langle y_1^2 \rangle_{t_2}}{2} + \frac{1}{2} - \frac{3}{2} \zeta  + \left(-\frac{1}{2} + \frac{3\zeta}{4} + \frac{9\zeta^2}{4}\right) \langle a_1^2 \rangle_{t_2} + \frac{1}{8} \left(1 - 9 \zeta^2\right) \langle a_1^4 \rangle_{t_2} \right) \, .
\end{align*}
We now compute the average values that appear above. Note that, as $y_1$ is periodic in~$t_2$, 
\begin{equation*}
    0 = \langle \partial_{t_2} y_1 \rangle_{t_2} = \eta \left(1 - \frac{\langle a_1^2 \rangle_{t_2}}{2}\right) \, ,
\end{equation*}
thus $\langle a_1^2 \rangle_{t_2} = 2$. Similarly, 
\begin{equation*}
    0 = \left\langle \partial_{t_2} (y_1^3) \right\rangle_{t_2} = 3 \left\langle y_1^2 \partial_{t_2} y_1 \right\rangle_{t_2} = 3 \eta \left\langle y_1^2 \left(1 - \frac{a_1^2}{2}\right)\right\rangle_{t_2} = 3\eta \left(\left\langle y_1^2 \right\rangle_{t_2} - \frac{\left\langle y_1^2 a_1^2 \right\rangle_{t_2}}{2}\right) \, ,
\end{equation*}
thus $\langle y_1^2 a_1^2 \rangle_{t_2} = 2 \langle y_1^2 \rangle_{t_2}$. Finally, 
\begin{align*}
    0 &= \left\langle \partial_{t_2} (a_1^2 y_1) \right\rangle_{t_2} = \eta \left\langle 2 a_1 (\partial_{t_2} a_1) y_1 + a_1^2 \partial_{t_2} y_1 \right\rangle_{t_2} = \eta \left\langle 2 a_1^2y_1^2 + a_1^2 \left(1 - \frac{a_1^2}{2}\right) \right\rangle_{t_2} \\
    &= \eta \left[2 \left\langle a_1^2 y_1^2 \right\rangle_{t_2} + \left\langle a_1^2  \right\rangle_{t_2} - \frac{\left\langle a_1^4 \right\rangle_{t_2}}{2}\right] = \eta \left[4 \left\langle y_1^2 \right\rangle_{t_2} + 2 - \frac{\left\langle a_1^4 \right\rangle_{t_2}}{2}\right]\, ,
\end{align*}
thus $\langle a_1^4 \rangle_{t_2} = 4 + 8 \langle y_1^2 \rangle_{t_2}$. Overall, we obtain 
\begin{align*}
    \partial_{t_3} E &= \eta^2 \left[ \frac{\langle y_1^2 \rangle_{t_2}}{2} + \frac{\langle y_1^2 \rangle_{t_2}}{2} + \frac{1}{2} - \frac{3}{2} \zeta  -1 + \frac{3\zeta}{2} + \frac{9 \zeta^2}{2} + (1 - 9 \zeta^2) \left(\frac{1}{2} + \langle y_1^2 \rangle_{t_2}\right) \right] \\
    &= \eta^2 (2 - 9 \zeta^2) \langle y_1^2 \rangle_{t_2}  \, .
\end{align*}

\begin{remark}
    In the above derivations, ensuring that $F = y_1 y_2 + \left(\frac{a_1}{2}-\frac{1}{a_1}\right) a_2$ is non-secular suffices to compute the variations of the energy $E$. However, ensuring that $y_2$ and $a_2$ are both non-secular would require a condition on an independent combination of $y_2$ and $a_2$. This additional condition would compute the phase shift of the self-stabilization fluctuations in the slow timescale $t_3$. This phase shift can be observed in Fig.~\ref{fig:minimal-example-3}, bottom plots. 
\end{remark}

\section{Derivation of Result~\ref{result:codim-one}}
\label{sec:derivation-codim-one}

        As before, we identify terms in the expansion of 
        \begin{align}
            \Delta_k x &= -\eta \partial_x f(x,y) \, , \label{eq:x-codim-one} \\
            \Delta_k y &= -\eta \nabla_y f(x,y) \, . \label{eq:y-codim-one}
        \end{align}
        The expansion of the discrete difference operator $\Delta_k$ is the same as in Eq.~\eqref{eq:lhs}. The expansion of the right-hand side also remains unchanged until Eq.~\eqref{eq:rhs-general}. Recall that for all $y \in Y$, we have $g(0,y) = \min g$. In particular, this implies that for all $p \geq 1$, $\nabla_y^p g(0,y) = 0$. Moreover, it also implies that for all $y \in Y$, $\partial_x g(0,y) = 0$. As a consequence, for all $p \geq 0$, $\nabla_y^p \partial_x g(0,y) = 0$. This simplifies many terms in Eq.~\eqref{eq:rhs-general}. Combining with the assumption $x_0 = 0$ and the notation $S(0,y) = \partial^2_{x} g(0,y)$, we obtain
        \begin{align}
            \label{eq:aux-2}
            \begin{split}
            -\eta \partial_x f(x,y) &= - \varepsilon^{1/2} \eta S(0,y_0) x_1 \\
            &\quad- \varepsilon \eta \left[S(0,y_0) x_2 + \frac{1}{2} \partial_x S(0,y_0) x_1^2 + x_1 \left\langle \nabla_y S(0,y_0), y_1 \right\rangle + \partial_x h(0,y_0)\right] \\
            &\quad+ o(\varepsilon) \, , \\
            -\eta \nabla_y f(x,y) &= - \varepsilon \eta \left[\frac{x_1^2}{2}  \nabla_y S(0,y_0) + \nabla_y h(0,y_0)\right] + o(\varepsilon) \, .
            \end{split}
        \end{align}
        In the derivations that follow, we do not write explicitly the evaluations at $(0,y_0)$ anymore. We can now proceed to the identification of terms in the expansion of Eqs.~\eqref{eq:x-codim-one}--\eqref{eq:y-codim-one}.

        \paragraph{Order $1$.}
        As in the derivation of Result~\ref{result:minimal}, this order gives $y_0 = y_0(t_2, t_3)$.

        \paragraph{Order $\varepsilon^{1/2}$.}
        The $x$ equation \eqref{eq:x-codim-one} gives
        \begin{equation*}
            \Delta_{t_1} x_1 = - \eta S x_1 \, .
        \end{equation*}
        We distinguish whether we are at the edge of stability.
        \begin{itemize}
            \item If $\eta S < 2$, recall that, by assumption from Sec.~\ref{sec:setting}, we also have $\eta S > 0$. Thus $x_1 \to 0$ as $t_1 \to \infty$. 
            \item If $\eta S = 2$, then $x_1 = (-1)^{t_1} a_1(t_2, t_3)$ for some $a_1$. 
        \end{itemize}
        As before, the $y$ equation \eqref{eq:y-codim-one} gives $y_1 = y_1(t_2, t_3)$ and $y_0 = y_0(t_3)$.

        \paragraph{Order $\varepsilon$.} The $x$ equation \eqref{eq:x-codim-one} gives
        \begin{equation*}
            \Delta_{t_1} x_2 + \partial_{t_2} x_1 + \partial_{t_2} \Delta_{t_1} x_1 = - \eta \left[S x_2 + \frac{1}{2} (\partial_x S) x_1^2 + x_1 \left\langle \nabla_y S, y_1 \right\rangle + \partial_x h\right] \,.
        \end{equation*}
        Again, we focus on understanding what happens at the edge of stability, i.e., when $\eta S = 2$. We obtain
        \begin{equation*}
            \Delta_{t_1} x_2 +  2 x_2 = (-1)^{t_1} \partial_{t_2} a_1 - \eta \left[\frac{1}{2} (\partial_x S) a_1^2 +  (-1)^{t_1} a_1 \left\langle \nabla_y S, y_1 \right\rangle + \partial_x h\right] \, .
        \end{equation*}
        By Lemma~\ref{lem:secular}, we must have $\partial_{t_2} a_1 = \eta a_1 \left\langle \nabla_y S, y_1 \right\rangle$ to avoid a secular term.

        We now turn to the $y$ equation \eqref{eq:y-codim-one}. It gives
        \begin{equation*}
            \Delta_{t_1} y_2 + \partial_{t_2} y_1 + \partial_{t_3} y_0 = - \eta \left[\frac{x_1^2}{2} \nabla_y S + \nabla_y h\right] \, .
        \end{equation*}
        We distinguish whether we are at the edge of stability.
        \begin{itemize}
            \item If $\eta S < 2$, then $x_1 \to 0$ as $t_1 \to \infty$. To prevent $y_2$ from developing a secular term on the timescale $t_1$, we must have
                $\partial_{t_2} y_1 + \partial_{t_3} y_0 = - \eta \nabla_y h$.
            Here, all quantities but $y_1$ are independent of $t_2$. To prevent a secular term, we must have 
            \begin{equation*}
                \partial_{t_3} y_0 = - \eta \nabla_y h \, .
            \end{equation*}
            In the relative interior of the stable region, the dynamics follow the gradient flow of $h$ constrained to the valley.

            \item If $\eta S = 2$, then we have $x_1 = (-1)^{t_1} a_1(t_2, t_3)$. We obtain
            \begin{equation*}
                \Delta_{t_1} y_2  = - \partial_{t_2} y_1 - \partial_{t_3} y_0- \eta \left[\frac{a_1^2}{2} \nabla_y S + \nabla_y h\right] \, .
            \end{equation*}
            The right-hand side is independent of $t_1$. To avoid a secular term, we must have 
            \begin{equation}
                \label{eq:aux-1}
                \partial_{t_2} y_1 = - \partial_{t_3} y_0- \eta \left[\frac{a_1^2}{2} \nabla_y S + \nabla_y h\right] \, .
            \end{equation}
            Differentiating the edge of stability condition $\eta S(0,y_0) = 2$ with respect to $t_3$, we obtain $\langle \partial_{t_3} y_0, \nabla_y S \rangle = 0$. This suggests taking the inner product of Eq.~\eqref{eq:aux-1} with $\nabla_y S$. Recall that we denote $u_1 = \langle y_1, \nabla_y S \rangle$. We obtain
            \begin{equation*}
                \partial_{t_2} u_1 = - \eta \left[\left\Vert \nabla_y S \right\Vert^2 \frac{a_1^2}{2} + \left\langle \nabla_y h, \nabla_y S  \right\rangle\right] \, .
            \end{equation*}
            This finishes the derivation of Eqs.~\eqref{eq:self-stab-codim-one}.

            As discussed below the result, the self-stabilization system~\eqref{eq:self-stab-codim-one} implies that $a_1$ is periodic in $t_2$ (for fixed $t_3$). Thus in the right hand side of Eq.~\eqref{eq:aux-1}, all quantities are independent of $t_2$ except for $a_1$ which is periodic in $t_2$. 

            By Lemma~\ref{lem:periodic}, to avoid $y_1$ developing a secular term in $t_2$, we must have 
            \begin{equation*}
                \partial_{t_3} y_0 = - \eta \left[\nabla_y h + \frac{\langle a_1^2 \rangle_{t_2}}{2}  \nabla_y S\right] \, .
            \end{equation*}
            \end{itemize}
            This completes the derivation of item \ref{it:codim-one-1} by taking $\sigma^2 = \langle a_1^2 \rangle_{t_2}$ if $\eta S =2$ and $\sigma^2 = 0$ if $\eta S < 2$, and thus completes the derivation of the result.

\section{Derivation of Result~\ref{result:general}}
\label{sec:derivation-general}

    As before, we identify terms in the expansion of
    \begin{align}
        \Delta_k x &= -\eta \nabla_x f(x,y) \, , \label{eq:x-general} \\
        \Delta_k y &= -\eta \nabla_y f(x,y) \, . \label{eq:y-general}
    \end{align}
    The expansion of the right hand side is the same as in Eq.~\eqref{eq:aux-2}, except for a few modifications due to the fact that $X$ is multidimensional: 
    \begin{itemize}
        \item $\partial_x$ is replaced by $\nabla_x$,
        \item $S(0,y) = \partial^2_{x} g(0,y) \in \R$ is replaced by $S(0,y) = \nabla^2_{x} g(0,y) \in \Sym_+ X$,
        \item $\partial_x S(0,y) \in \R$ is replaced by $D_x S(0,y) : X \to \Sym X$,
        \item $\nabla_y S(0,y) \in Y$ is replaced by $D_y S(0,y) : Y \to \Sym X$. 
    \end{itemize} 
    We obtain:
    \begin{align*}
        -\eta \nabla_x f(x,y) &= - \varepsilon^{1/2} \eta S(0,y_0) x_1 \\
        &\quad- \varepsilon \eta \left[S(0,y_0) x_2 + \frac{1}{2} D_x S(0,y_0)[x_1] x_1 + D_y S(0,y_0)[y_1] x_1 + \nabla_x h(0,y_0)\right] \\
        &\quad+ o(\varepsilon) \, , \\
        -\eta \nabla_y f(x,y) &= - \varepsilon \eta \left[\frac{1}{2} D_y S(0,y_0)^*[x_1 x_1^\top] + \nabla_y h(0,y_0)\right] + o(\varepsilon) \, .
    \end{align*}
    In the derivations that follow, we do not write explicitly the evaluations at $(0,y_0)$ anymore. We can now proceed to the identification of terms in the expansion of Eqs.~\eqref{eq:x-general}--\eqref{eq:y-general}. 

    \paragraph{Order $1$.} As before, this order gives $y_0 = y_0(t_2, t_3)$.

    \paragraph{Order $\varepsilon^{1/2}$.} The $x$ equation \eqref{eq:x-general} gives
    \begin{equation*}
        \Delta_{t_1} x_1 = - \eta S x_1 \, .
    \end{equation*}
    We distinguish directions depending on whether they are critical. Recall that $C = C(t_3) = \ker(2 \Id_X - \eta S)$ denotes the critical subspace. We use a spectral decomposition of $\eta S$. By assumption from Sec.~\ref{sec:setting}, all eigenvalues are positive. The eigenvectors associated the eigenvalue $2$ form a basis of $C$, while the other eigenvectors are associated with eigenvalues in $(0,2)$ and form a basis of $C^\perp$. 
    
    Denoting $x_1^\parallel = P_C x_1$ and $x_1^\perp = P_{C^\perp} x_1$, this gives $x_1^\perp \xrightarrow[t_1 \to \infty]{} 0$, with $\sum_{t_1=0}^{+\infty} \|x_1^\perp(t_1, t_2, t_3)\| < +\infty$ and $x_1^\parallel = (-1)^{t_1} a_1$ for some $a_1 = a_1(t_2, t_3) \in C$.

    As before, the $y$ equation \eqref{eq:y-general} gives $y_1 = y_1(t_2, t_3)$ and $y_0 = y_0(t_3)$.

\paragraph{Order $\varepsilon$.} The $x$ equation \eqref{eq:x-general} gives
    \begin{equation*}
        \Delta_{t_1} x_2 + \partial_{t_2} x_1 + \partial_{t_2} \Delta_{t_1} x_1 = - \eta \left[S x_2 + \frac{1}{2} D_x S[x_1] x_1 + D_y S[y_1] x_1 + \nabla_x h\right] \, .
    \end{equation*}
    We focus on the critical directions, i.e., we apply $P_C$ to both sides of the equation. Denoting $x_2^\parallel = P_C x_2$, we obtain
    \begin{equation*}
        \Delta_{t_1} x_2^\parallel + \partial_{t_2} x_1^\parallel + \partial_{t_2} \Delta_{t_1} x_1^\parallel = - \eta \left[P_C S 
        x_2 + \frac{1}{2} P_C D_x S[x_1] x_1 + P_C D_y S[y_1] x_1 + P_C \nabla_x h\right] \, .
    \end{equation*}
Recall that $C$ is the eigenspace of $\eta S$ associated with the eigenvalue $2$, thus $\eta P_C S = 2 P_C$. Moreover, we decompose $x_1 = \iota_C x_1^\parallel + \iota_{C^\perp} x_1^\perp = (-1)^{t_1} \iota_C a_1 + \iota_{C^\perp} x_1^\perp$. This gives
    \begin{align*}
        \Delta_{t_1} x_2^\parallel + 2 x_2^\parallel &= (-1)^{t_1} \partial_{t_2} a_1 - \eta \bigg[\frac{1}{2} P_C D_x S[\iota_Ca_1] \iota_C a_1 + \frac{1}{2}(-1)^{t_1} P_C D_x S[\iota_{C^\perp} x_1^\perp] \iota_Ca_1 \\
        &\hspace{8em}+ \frac{1}{2}(-1)^{t_1} P_C D_x S[\iota_C a_1] \iota_{C^\perp} x_1^\perp + \frac{1}{2}P_C D_x S[\iota_{C^\perp} x_1^\perp] \iota_{C^\perp} x_1^\perp \\
        &\hspace{8em}+ (-1)^{t_1} P_C D_y S[y_1] \iota_Ca_1 + P_C D_y S[y_1] \iota_{C^\perp} x_1^\perp + P_C \nabla_x h\bigg] \, .
    \end{align*}
    The right hand side depends on $t_1$ only through the $(-1)^{t_1}$ terms and $x_1^\perp$ terms, but the latter are vanishing as $t_1 \to \infty$, with $\sum_{t_1 = 0}^{+\infty} \Vert x_1^\perp \Vert < +\infty$. This suggests studying the following generalization of Lemma~\ref{lem:secular}.

    \begin{lem}
        \label{lem:secular-general}
        Let $\alpha, \beta \in \R$ and $\gamma = \gamma(t_1) \in \R$ such that $\sum_{t_1=0}^{+\infty} |\gamma(t_1)| < +\infty$. The solutions $\varphi = \varphi(t_1)$ of the difference equation
        \begin{equation*}
            \Delta_{t_1} \varphi + 2 \varphi = \alpha + \beta (-1)^{t_1} + \gamma(t_1)
        \end{equation*}
        are 
        \begin{equation*}
            \varphi = \frac{\alpha}{2} + a (-1)^{t_1} - \beta t_1 (-1)^{t_1} - \sum_{s=0}^{t_1-1} (-1)^{t_1+s} \gamma(s) \, , \qquad a \in \R \, .
        \end{equation*} 
        In particular, there exists a bounded solution if and only if $\beta = 0$.
    \end{lem}

    \begin{proof}
        As $(-1)^{t_1}$ is solution to the homogeneous equation, we seek solutions of the form $\varphi = A(-1)^{t_1}$ for some $A = A(t_1)$. We compute
        \begin{align*}
            \Delta_{t_1} \varphi + 2 \varphi &= \varphi(t_1 + 1) - \varphi(t_1) + 2 \varphi(t_1) = \varphi(t_1 + 1) + \varphi(t_1) \\
            &= A(t_1 + 1)(-1)^{t_1 + 1} + A(t_1)(-1)^{t_1} = -(-1)^{t_1} \Delta_{t_1} A \, .
        \end{align*}
        The function $\varphi = A(-1)^{t_1}$ solves the difference equation if and only if $A$ solves
        \begin{equation*}
            \Delta_{t_1} A = - \alpha (-1)^{t_1} - \beta - (-1)^{t_1} \gamma(t_1) \, .
        \end{equation*}
        The discrete primitives are 
        \begin{equation*}
            A = a + \frac{\alpha}{2} (-1)^{t_1} - \beta t_1 - \sum_{s=0}^{t_1-1} (-1)^s \gamma(s) \, , \qquad a \in \R \, .
        \end{equation*}
        This gives the result.
    \end{proof}
    Applying Lemma~\ref{lem:secular-general} to the equation above, we obtain that to avoid a secular term, we must have 
    \begin{equation*}
        \partial_{t_2} a_1 = \eta P_C D_y S[y_1] \iota_C a_1 = \eta \Psi[y_1] a_1 = \eta U_1 a_1 \, .
    \end{equation*}

    We now turn to the $y$ equation \eqref{eq:y-general}. It gives
    \begin{equation*}
        \Delta_{t_1} y_2 + \partial_{t_2} y_1 + \partial_{t_3} y_0 = - \eta \left[\frac{1}{2} (D_y S)^*[x_1 x_1^\top] + \nabla_y h\right] \, .
    \end{equation*}
    Again, we decompose $x_1 = (-1)^{t_1} \iota_C a_1 + \iota_{C^\perp} x_1^\perp$. 
    \begin{align*}
        &\Delta_{t_1} y_2 + \partial_{t_2} y_1 + \partial_{t_3} y_0 \\
        &\qquad= - \eta \bigg[\frac{1}{2} (D_y S)^*[\iota_C a_1 a_1^\top P_C] + \frac{(-1)^{t_1}}{2} (D_y S)^*\left[\iota_C a_1 x_1^{\perp \top} P_{C^\perp} + \iota_{C^\perp} x_1^\perp a_1^\top P_C\right] \\
        &\qquad\qquad\qquad+ \frac{1}{2} (D_y S)^*[\iota_{C^\perp} x_1^\perp x_1^{\perp \top} P_{C^\perp}] + \nabla_y h\bigg] \, .
    \end{align*}
    Here, the right hand side depends on $t_1$ only through the $(-1)^{t_1}$ terms and $x_1^\perp$ terms, but the latter are vanishing as $t_1 \to \infty$. To avoid a secular term, we must have
    \begin{equation*}
        \partial_{t_2} y_1 + \partial_{t_3} y_0 = - \eta \left[\frac{1}{2} (D_y S)^*[\iota_C a_1 a_1^\top P_C] + \nabla_y h\right] \, .
    \end{equation*}
    Note that for $M \in \Sym C$, $\Psi^*[M] = (D_y S)^*[\iota_CMP_C]$. Thus we have 
    \begin{equation}
        \label{eq:aux-3}
        \partial_{t_2} y_1 + \partial_{t_3} y_0 = - \eta \left[\frac{1}{2} \Psi^*[a_1 a_1^\top] + \nabla_y h\right] \, .
    \end{equation}
    Here, we want to use that the evolution of $S(0,y_0)$ is constrained in the critical subspace~$C$. We use the following lemma. 

    \begin{lem}
        Let $t \mapsto M(t) \succcurlyeq 0$ be a smooth curve of positive semidefinite matrices, $P = P(t)$ denote the orthogonal projection onto $\ker M(t)$, and $\iota = \iota(t) = P^*$ the corresponding inclusion operator. Then, for all $t$, we have $P \partial_t M \iota = 0$.
    \end{lem}

    \begin{proof}
        Let $a \in \ker M(t)$. Then $a^\top M(t) a = 0$. Moreover, for all $s$, we have $a^\top M(s) a \geq 0$. We thus have, for all $a \in \ker M(t)$, $\partial_t (a^\top M(t) a) = a^\top \partial_t M(t) a = 0$. 

        Now, consider $a, b \in \ker M(t)$. Then, by polarization ($\partial_t M(t)$ is symmetric),
        \begin{equation*}
            2 a^\top \partial_t M(t) b = (a + b)^\top \partial_t M(t) (a + b) - a^\top \partial_t M(t) a - b^\top \partial_t M(t) b = 0 \, .
        \end{equation*}
        Thus $P \partial_t M \iota = 0$.
    \end{proof}

    Here, we apply the lemma to $M(t_3) = 2 \Id_X - \eta S(0,y_0(t_3))$. We obtain
    \begin{align*}
        0 = P_C \partial_{t_3} S \iota_C = P_C D_y S[\partial_{t_3} y_0] \iota_C = \Psi[\partial_{t_3} y_0] \, .
    \end{align*}
    This suggests evaluating $\Psi$ at Eq.~\eqref{eq:aux-3}. We obtain 
    \begin{equation*}
        \partial_{t_2} U_1 = - \eta \left[\frac{1}{2} \Psi\Psi^*[a_1 a_1^\top] + \Psi[\nabla_y h]\right] \, .
    \end{equation*}
    This completes the derivation of the self-stabilization system~\eqref{eq:self-stab-general}. 

    We do not know whether the self-stabilization system induces periodic solutions in $t_2$ when more than one eigenvalue is on the edge of stability. To give a meaning to the average $\langle . \rangle_{t_2}$, we take the limit $T_2 \to \infty$ of the average over $[0, T_2]$ (if the limit exists). For instance, to avoid a secular term, $y_1$ is uniformly bounded in $t_2$, thus we have 
    \begin{equation*}
        \langle \partial_{t_2} y_1 \rangle_{t_2} = \lim_{T_2 \to \infty} \frac{1}{T_2} \int_0^{T_2} \partial_{t_2} y_1  \diff t_2 = \lim_{T_2 \to \infty} \frac{y_1(T_2, t_3) - y_1(0, t_3)}{T_2} = 0 \, .
    \end{equation*}
    Taking the average of Eq.~\eqref{eq:aux-3} in $t_2$, we obtain
    \begin{equation*}
        \partial_{t_3} y_0 = - \eta \left[\nabla_y h + \frac{1}{2} \langle \Psi^*[a_1 a_1^\top] \rangle_{t_2}\right] \, .
    \end{equation*}
    Provided that the average $\langle a_1 a_1^\top \rangle_{t_2}$ is well-defined, we have 
    \begin{equation*}
        \langle \Psi^*[a_1 a_1^\top] \rangle_{t_2} = \Psi^*[\langle a_1 a_1^\top \rangle_{t_2}] = (D_y S)^*[\iota_C \langle a_1 a_1^\top \rangle_{t_2} P_C] \, , 
    \end{equation*}
    thus the result holds with $\Sigma = \iota_C \langle a_1 a_1^\top \rangle_{t_2} P_C$, which is complementary to $2 \Id_X - \eta S$ (by definition of $C = \ker(2 \Id_X - \eta S)$). If we do not make the assumption that the average $\langle a_1 a_1^\top \rangle_{t_2}$ is well-defined, we use that $a_1$ is bounded in $t_2$ (it is non-secular) and thus the averages $\frac{1}{T_2} \int_0^{T_2} a_1 a_1^\top \diff t_2$ have an accumulation point $\widetilde{\Sigma} \in \Sym_+ C$ as $T_2 \to \infty$. We then have $\langle \Psi^*[a_1 a_1^\top] \rangle_{t_2} = \Psi^*[\widetilde{\Sigma}]$, and we can define $\Sigma = \iota_C \widetilde{\Sigma} P_C \in \Sym_+ X$. This gives the derivation of item \ref{it:general-1} and completes the derivation of Result~\ref{result:general}.

\section{Derivations of Corollary~\ref{coro:general}}
    \label{sec:derivation-coro-general}

        We compute
        \begin{align*}
            f(x,y) &= g(x,y) + \varepsilon h(x,y) \\
            &= g(\varepsilon^{1/2} x_1 + \varepsilon x_2 + o(\varepsilon), y_0 + \varepsilon^{1/2} y_1+ \varepsilon y_2 + o(\varepsilon)) + \varepsilon h(o(1), y_0 + o(1)) \\
            &= g(0,y_0) + \left\langle \nabla_x g(0,y_0) ,\varepsilon^{1/2} x_1 + \varepsilon x_2 \right\rangle + \left\langle \nabla_y g(0,y_0) , \varepsilon^{1/2} y_1 + \varepsilon y_2 \right\rangle \\
            &\qquad+ \frac{1}{2} \left\langle \varepsilon^{1/2} x_1, \nabla^2_x g(0,y_0) \varepsilon^{1/2} x_1 \right\rangle + \left\langle \varepsilon^{1/2} x_1, \nabla_x \nabla_y g(0,y_0) \varepsilon^{1/2} y_1 \right\rangle \\
            &\qquad+ \frac{1}{2} \left\langle \varepsilon^{1/2} y_1, \nabla^2_y g(0,y_0) \varepsilon^{1/2} y_1 \right\rangle + \varepsilon h(0,y_0) + o(\varepsilon) \, . 
        \end{align*}
        As in the derivations above, we have $\nabla_x g(0,y_0) = 0$, $\nabla_y g(0,y_0) = 0$, $\nabla^2_{x} g(0,y_0) = S(0,y_0)$, $\nabla_x \nabla_y g(0,y_0) = 0$, and $\nabla^2_{y} g(0,y_0) = 0$. Moreover, we have $g(0,y_0) = \min g$. We obtain
        \begin{align*}
            &f(x,y) = \min g + \varepsilon f_2 + o(\varepsilon) \, , 
            &&f_2 = h(0,y_0) + \frac{1}{2} \langle x_1, S(0,y_0) x_1 \rangle \, .
        \end{align*}
        Here, we decompose $x_1 = (-1)^{t_1} \iota_C a_1 + \iota_{C^\perp} x_1^\perp$. As $x_1^\perp \to 0$ as $t_1 \to \infty$, and $C$ is the eigenspace of $S(0,y_0)$ associated to the eigenvalue $2/\eta$, we have
        \begin{align*}
            f_2 &\xrightarrow[t_1 \to \infty]{} h(0, y_0) + \frac{1}{2} \langle \iota_Ca_1, S(0,y_0) \iota_C a_1 \rangle =  h(0, y_0) + \frac{\Vert a_1 \Vert^2}{\eta} \, .
        \end{align*}

\end{document}